\documentclass[final,12pt]{colt2018} 
\usepackage{amsmath}
\usepackage{amssymb}
\usepackage{algorithm}
\usepackage{algorithmic}
\usepackage[usenames,dvipsnames]{pstricks}
\usepackage{pst-grad} 
\usepackage{pst-plot} 
\usepackage{commath}
\usepackage{booktabs}

\def\calF{\mathcal{F}}
\def\calG{\mathcal{G}}
\def\R{\mathbb{R}}
\def\E{\mathbb{E}}
\def\P{\mathbb{P}}
\def\calX{\mathcal{X}}
\def\calY{\mathcal{Y}}
\def\calH{\mathcal{H}}
\def\calO{\mathcal{O}}

\def\calZ{\mathcal{Z}}
\DeclareMathOperator{\argmin}{argmin}

\DeclareMathOperator{\HT}{P}

\DeclareMathOperator{\sign}{sign}
\DeclareMathOperator{\err}{err}
\DeclareMathOperator{\polylog}{polylog}
\DeclareMathOperator{\poly}{poly}

\def\cf{C_3}
\def\ct{C_5}
\def\ce{C_4}
\def\cs{C_6}
\def\cse{C_7}
\def\cfo{C_8}
\def\cn{C_9}

\title[Efficient active learning of sparse halfspaces]{Efficient active learning of sparse halfspaces}
\usepackage{times}

\coltauthor{\Name{Chicheng Zhang} \Email{chicheng.zhang@microsoft.com}
\\
\addr Microsoft Research \\
641 6th Avenue \\
New York, NY, 10011 \\
USA
}

\begin{document}

\maketitle

\begin{abstract}
We study the problem of efficient PAC active learning of homogeneous linear classifiers (halfspaces) in $\R^d$, where the goal is to learn
a halfspace with low error using as few label queries as possible.
Under the extra assumption that there is a $t$-sparse halfspace that
performs well on the data ($t \ll d$),
we would like our active learning algorithm to be {\em attribute efficient}, i.e. to have label requirements sublinear in $d$.
In this paper, we provide a computationally efficient algorithm that achieves this goal.
Under certain distributional assumptions on the data, our algorithm achieves a label complexity of $O(t \cdot \polylog(d, \frac 1 \epsilon))$.
In contrast, existing algorithms in this setting are either computationally inefficient, or subject to label requirements
polynomial in $d$ or $\frac 1 \epsilon$.
\end{abstract}


\section{Introduction}



Active learning is a machine learning paradigm that aims at reducing label requirements through interacting with labeling oracles~\citep{S10}. The learner is given a distribution from which it can draw unlabeled examples,
and a labeling oracle from which it can query labels interactively. This is in contrast with passive learning, where labeled examples are drawn from distributions directly.
Using the ability to adaptively query labels, an active learning
algorithm can avoid querying the labels it has known before, thus substantially reducing label requirements.
In the PAC active learning model~\citep{V84,KSS94,BBL09,H14}, the performance of an active learner is measured by its label complexity, i.e. the number of label requests to satisfy an error requirement $\epsilon$ with high probability.


There have been many exciting works on active halfspace learning in the literature.
In this setting, the instances are in $\R^d$, and the labels are from $\{-1,+1\}$. The goal is to learn a classifier from $\calH = \{\sign(w \cdot x): w \in \R^d \}$, the class of homogeneous linear classifiers, to predict labels from instances.
Efficient active halfspace learning algorithms that work under different distributional assumptions have been proposed. Some of these algorithms are computationally efficient, and enjoy
information theoretically optimal label complexities~\citep{DKM05, BBZ07, ABL17, HKY15, ABHU15, YZ17}, that is,
$O(d \ln\frac 1 \epsilon)$ in terms of $d$ and $\epsilon$ \citep[See e.g.][for an $\Omega(d \ln\frac 1 \epsilon)$ lower bound]{KMT93}.
On the other hand, a line of work on attribute efficient learning \citep{B90} shows that one can in fact learn faster
when the target classifier is {\em sparse}, i.e. it depends only on a few of the input features.
In the problem of active halfspace learning, one can straightforwardly apply existing results to achieve attribute efficiency.
For instance, consider running the algorithm of ~\cite{ZC14} with concept class $\calH_t$, the set of $t$-sparse linear classifiers. Under certain distributional assumptions, \cite{ZC14}'s algorithm achieves label complexities of order $O(t \ln d \ln\frac 1 \epsilon)$. However, such algorithms are computationally inefficient: they require solving empirical 0-1 loss minimization with respect to $\calH_t$, which is NP-hard even in the realizable setting~\citep{N95}.


The results above raise the following question: are there active learning algorithms that learn linear classifiers in an attribute and computationally efficient manner?
A line of work on one-bit compressed sensing~\citep{BB08}, partially answers this question. They show that when the learning algorithm is allowed to synthesize instances to query their labels (also known as the membership query model~\citep{A88}, abbrev. MQ), it is possible to approximately recover the target halfspace using a near-optimal number of $\tilde{O}(t (\ln d + \ln \frac 1 \epsilon))$ queries~\citep{HB11}.
However, when applied to active learning in the PAC model, these results have strong distributional
requirements.
For instance, the algorithm of~\cite{HB11} requires the unlabeled distribution to have a constant probability to observe elements in the discrete set $\{-1,0,+1\}^d$.

In the PAC setting, recent work of~\cite{ABHZ16} proposes attribute and computationally efficient active halfspace learning algorithms, under the assumption that the unlabeled distribution is isotropic log-concave~\citep{LV07}. In the $t$-sparse $\Omega(\epsilon)$-adversarial noise setting, where all but an $\Omega(\epsilon)$ fraction of examples agree with some $t$-sparse linear classifier (see also Definition~\ref{def:an}), their algorithm has a label complexity of $\tilde{O}(\frac{t}{\epsilon^2} )$.
In the $t$-sparse $\eta$-bounded noise setting, where each label is generated by some underlying $t$-sparse linear classifier and then flipped with probability at most a constant $\eta \in [0,\frac12)$ (see also Definition~\ref{def:bn}), their algorithm has a label complexity of $\tilde{O}((\frac{t}{\epsilon})^{O(1)} )$. Compared to those achieved by computationally inefficient algorithms (e.g. ~\cite{ZC14} discussed above), these label complexity bounds are suboptimal, in that they do not have a logarithmic dependence on $\frac 1 \epsilon$.


In this paper, we give an algorithm that combines the advantages of~\cite{ZC14} and~\cite{ABHZ16}, achieving computational efficiency and $\tilde{O}(t \polylog(d, \frac 1 \epsilon))$ label complexity simultaneously, under certain distributional assumptions on the data.
Specifically, our algorithm works if the unlabeled distribution is isotropic log-concave, and has the following guarantee.
If one of the two conditions below is true:
\begin{enumerate}
\item the $t$-sparse $\mu_1\epsilon$-adversarial noise condition holds (see Definition~\ref{def:an}), where $\mu_1 > 0$ is some numerial constant;

\item the $t$-sparse $\mu_2$-bounded noise condition holds (see Definition~\ref{def:bn}), where $\mu_2 > 0$ is some numerical constant,

\end{enumerate}
then, with high probability, the algorithm outputs a halfspace with excess error at most $\epsilon$, and queries at most $O(t (\ln d + \ln \frac 1 \epsilon)^3 \ln \frac 1 \epsilon )$ labels. As a corollary, if there is a $t$-sparse linear classifier that
agrees with all the labeled examples drawn from the distribution (see Definition~\ref{def:r}), the algorithm also achieves a label complexity of $O(t (\ln d + \ln \frac 1 \epsilon)^3 \ln \frac 1 \epsilon )$. In the next section, we give a detailed comparison between
our results and related results in the literature.

From a technical perspective, our algorithm combines the margin-based
framework of~\cite{BBZ07, BL13} with iterative hard thresholding~\citep{BD09, GK09}, a technique well-studied in compressed sensing~\citep{CT06,D06}. Our analysis is based on sharp uniform concentration bounds of hinge losses over linear predictors in $\ell_1$ balls in the label query regions, which is in turn built upon classical Rademacher complexity bounds for linear prediction~\citep{KST09}.

\section{Related work}

\paragraph{Attribute efficient active learning of halfspaces.}
There is a rich body of theoretical literature on active learning of general concept classes
in the PAC setting~\citep{D11, H14}. For the problem of active halfspace learning, sharp distribution-dependent label complexity results are known,
in terms of e.g. the splitting index~\citep{D05}, or the disagreement coefficient~\citep{H07}.
Direct applications of these results (without taking advantage of sparsity assumptions)
yield algorithms with label complexities at least $\Omega(d \ln \frac 1 \epsilon)$~\citep{KMT93}.
To make these algorithms attribute efficient, a natural modification is to consider concept class
$\calH_t$, the set of $t$-sparse linear classifiers.
It is well known that $\calH_t$ has VC dimension
$O(t \ln d)$. In conjunction with existing results in the active learning
literature, this observation immediately yields attribute efficient active
learning algorithms. For example, when the unlabeled distribution is isotropic log-concave,
an application of~\cite{ZC14}'s algorithm with $\calH_t$ yields a label complexity
of $O(t \ln d \ln \frac 1 \epsilon)$ in the $t$-sparse realizable setting, and gives
$O(t \ln d \cdot (\ln \frac 1 \epsilon+\frac{\nu^2}{\epsilon^2}))$ and
$O(\frac{t\ln d}{(1-2\eta)^2} \ln \frac 1 \epsilon)$
label complexities in the $t$-sparse $\nu$-adversarial noise and $t$-sparse $\eta$-bounded noise settings.\footnote{To see this, note that the $\phi(\cdot,\cdot)$ function
defined in~\cite{ZC14} with respect to $\calH_t$ can be bounded as: $\phi(r,\xi) \leq O(r \ln \frac{r}{\xi})$, as $\calH_t$ is a subset of $\calH$. Theorem 4 of \cite{ZC14} now applies.}
However, these algorithms require solving
empirical 0-1 loss minimization subject to sparsity constraints, which is computationally intractable in general~\citep{N95}.
The only attribute and computationally efficient PAC active learning algorithms we are aware of are in~\cite{ABHZ16}.  Specifically, under the $t$-sparse $\Omega(\epsilon)$-adversarial noise setting, \cite{ABHZ16} gives an efficient algorithm with label complexity $\tilde{O}(\frac{t}{\epsilon^2}\polylog(d,\frac 1 \epsilon))$. Under the $t$-sparse $\eta$-bounded noise setting, ~\cite{ABHZ16} gives an efficient algorithm with label complexity $\tilde{O}((\frac t \epsilon)^{O(\frac 1 {(1-2\eta)^2})})$.





The notion of attribute efficient learning algorithms is initially studied in the pioneering works of~\cite{L87,B90}.
\cite{L87} considers attribute efficient online
learning of linear classifiers, with an application to learning disjunctions that depends on only $t$ attributes.
The algorithm incurs a mistake bound of $O(t \ln d)$, which can be of substantially lower order than $O(d)$ when $t$ is small.
\cite{B90} considers an online learning model where the feature space is infinite dimensional,
and each instance shown has a bounded number of nonzero attributes.
It gives efficient algorithms that learn $k$-CNFs and disjunctions
with finite mistake bounds in this setting.
\cite{S00, KS06, STT12} study attribute efficient learning of decision lists and analyzes the
tradeoff between running time and mistake bound.
\cite{LS07} shows that, if the unlabeled distribution is unconcentrated over $\{-1,1\}^d$, then there
is an algorithm that learns $t$-sparse linear classifiers with a sample complexity of $\poly(t, \ln d, 2^{O(\epsilon^{-2})})$. \cite{F07} gives algorithms for attribute efficient learning parity and DNFs
in the membership query model.


\paragraph{One-bit compressed sensing.} The line of work on one-bit compressed sensing~\citep{BB08} is closely related to our problem setup. In this setting,
there is a unknown $t$-sparse vector $u \in \R^d$, and the algorithm can make measurements of $u$ using vectors $x \in \R^d$ and receives (possibly noisy) values of $\sign(u \cdot x)$.
Note that different from standard compressed sensing~\citep{CT06,D06}, the measurement results of one-bit compressed sensing are {\em quantized} versions of $(u \cdot x)$'s (i.e. they lie in $\{-1,+1\}$ as opposed to $\R$).
The goal is to approximately recover $u$ up to scaling with a few (ideally, $O(t \ln d)$) measurements.
 In the non-adaptive setting, the measurement vector
$x$'s are chosen at the beginning, while in the adaptive setting, the measurement vector $x$'s can be chosen sequentially,
based on past observations.
The problem of adaptive one-bit compressed sensing is therefore equivalent to attribute efficient
active halfspace learning in the membership query model~\citep{A88}.
We remark that active learning in the PAC model is more challenging than in the membership model, in that the learner has to query the labels of the unlabeled examples it has drawn.


~\cite{JLBB13} gives an algorithm that has robust recovery guarantees, however it is based on computationally-intractable $\ell_0$ minimization. Inspired by the count sketch data structure~\citep{CCF02}, ~\cite{HB11} proposes an efficient procedure that recovers the support of $u$ using $O(t \ln d)$ queries, and has strong noise tolerance properties. In conjunction with efficient full-dimensional active halfspace learning algorithms~\citep{DKM05,ABL17,CHK17,YZ17}, this procedure
 yields efficient algorithms that have label complexities of $O(t (\ln d + \ln \frac 1 \epsilon ))$
(resp. $O(t (\ln d + \ln \frac 1 \epsilon))$, $O(\frac{t}{(1-2\eta)^2} (\ln d + \ln \frac 1 \epsilon ))$) in the $t$-sparse realizable setting (resp. $t$-sparse $\Omega(\epsilon)$-adversarial noise setting, $t$-sparse $\eta$-bounded noise setting).
~\cite{GNJN13, ABK17} gives upper and lower bounds for {\em universal} one-bit compressed sensing, that is, the same set of measurements can be used to approximately recover {\em any} underlying $t$-sparse signal. In this setting,~\cite{ABK17} shows that, perhaps surprisingly, the number of measurements necessary and sufficient for support recovery is $\tilde{\Theta}(t^2 \ln d)$, as opposed to $\Theta(t \ln d)$ in the non-universal setting.
 ~\cite{PV13a} proposes a linear programming based algorithm that works in the $t$-sparse realizable setting, and has a measurement complexity of $\tilde{O}(\frac{t}{\epsilon^5})$,
based on a new tool named random hyperplane tessellations. ~\cite{L16} gives a support recovery algorithm that tolerates bounded noise, 
using $\alpha$-stable random projections.
\cite{PV13b} proposes a convex programming based algorithm that works in the $t$-sparse $\Omega(\epsilon^2)$-adversarial noise model,
and has a measurement complexity of $\tilde{O}(\frac{t}{\epsilon^{12}})$.

Works on one-bit compressed sensing under the symmetric noise condition has been studied in the literature~\citep{PV13b, ZYJ14, CB15, ZG15}. In this model, it is assumed that there is a known function $g$, such that for all $x$, $\E[y|x] = g(u \cdot x)$. This assumption captures some realistic scenarios, but is nevertheless strong: it requires any two examples that have the same projection on $u$ to have the same conditional label distribution. In contrast, the $t$-sparse adversarial noise and the $t$-sparse bounded noise conditions allow heterogeneous noise levels, even among examples that have the same projection on $u$.
In this setting, the state of the art result of \cite{ZYJ14} gives an nonadaptive algorithm with $O(\frac{t \ln d}{ \epsilon^2})$. It also proposes an adaptive algorithm that works in same setting, achieving a label complexity bound of $O(\min(\frac{t \ln d}{\epsilon^2}, \frac{t\sqrt{d} \ln d}{\epsilon}))$, which is sometimes lower than that of the nonadaptive algorithm.
The special case of Gaussian noise before quantization has been studied extensively, i.e. given $x$, the label $y$ is generated by the formula $y = \sign(u \cdot x + n)$, where $n$ is a Gaussian random variable. \cite{GNR10} shows that when $u$ has a large dynamic range (the absolute value of the ratio between $u$'s largest and smallest nonzero elements in magnitude), adaptive approaches require fewer measurements to identify the support of $u$ than nonadaptive approaches.

We provide a detailed comparison between our work and the results most closely related to ours in Tables~\ref{tab:comp-r}, \ref{tab:comp-an}, and \ref{tab:comp-bn}.

\begin{table}[t]
\centering
\begin{tabular}{llll}
\toprule
Algorithm & Model & Label complexity & Efficient? \\
\midrule
\begin{tabular}{@{}c@{}}\cite{HB11}\\ with \cite{DKM05} \end{tabular} & MQ & $\tilde{O}(t (\ln d + \ln \frac 1 \epsilon))$ & Yes \\
\cite{D05} & PAC & $\tilde{O}(t (\ln d + \ln \frac 1 \epsilon))$ & No \\
\cite{ABHZ16} & PAC & $\tilde{O}(\frac{t}{\epsilon^2} \polylog(d,\frac 1 \epsilon) )$ & Yes \\
Our work & PAC & $\tilde{O}(t \polylog(d,\frac 1 \epsilon) )$ & Yes \\
\bottomrule
\end{tabular}
\caption{A comparison of algorithms for active learning of halfspaces in the $t$-sparse realizable setting (Definition~\ref{def:r}); all the PAC algorithms above work under isotropic log-concave distributions.}
\label{tab:comp-r}
\end{table}

\begin{table}[t]
\centering
\begin{tabular}{lllll}
\toprule
Algorithm & Model & Noise tolerance  & Label complexity & Efficient? \\
\midrule
\begin{tabular}{@{}c@{}}\cite{HB11}\\ with \cite{ABL17} \end{tabular}& MQ & $\nu = \Omega(\epsilon)$ & $\tilde{O}(t (\ln d + \ln \frac 1 \epsilon))$ & Yes \\
\cite{ZC14} & PAC & $\nu = \Omega(\epsilon)$ & $\tilde{O}(t \ln d \ln \frac 1 \epsilon))$ & No \\
\cite{PV13b} & PAC  & $\nu = \Omega(\epsilon^2)$ & $\tilde{O}(\frac{t\ln d}{\epsilon^{12}})$ & Yes \\
\cite{ABHZ16} & PAC & $\nu = \Omega(\epsilon)$ & $\tilde{O}(\frac{t}{\epsilon^2} \polylog(d,\frac 1 \epsilon) )$ & Yes \\
Our work & PAC & $\nu = \Omega(\epsilon)$ & $\tilde{O}(t \polylog(d,\frac 1 \epsilon) )$ & Yes \\
\bottomrule
\end{tabular}
\caption{A comparison of algorithms for active learning of halfspaces in the $t$-sparse $\nu$-adversarial noise setting (Definition~\ref{def:an});  all the PAC algorithms above work under isotropic log-concave distributions.}
\label{tab:comp-an}
\end{table}

\begin{table}[t]
\centering
\begin{tabular}{lllll}
\toprule
Algorithm & Model & Noise tolerance & Label complexity & Efficient? \\
\midrule
\begin{tabular}{@{}c@{}}\cite{HB11}\\ with \cite{CHK17} \end{tabular} & MQ & $\eta \in [0,\frac 1 2)$ & $\tilde{O}(\frac{t}{(1-2\eta)^2} (\ln d + \ln \frac 1 \epsilon))$ & Yes \\
\cite{ZC14} & PAC & $\eta \in [0,\frac 1 2)$ & $\tilde{O}(\frac{t}{(1-2\eta)^2} \ln d \ln \frac 1 \epsilon))$ & No \\
\cite{ABHZ16} & PAC & $\eta \in [0,\frac 1 2)$ & $\tilde{O}((\frac{t}{\epsilon})^{O(\frac{1}{(1-2\eta)^2})} )$ & Yes \\
Our work & PAC & $\eta \in [0, \Omega(1))$ & $\tilde{O}(t \polylog(d,\frac 1 \epsilon) )$ & Yes \\
\bottomrule
\end{tabular}
\caption{A comparison of algorithms for active learning of halfspaces in the $t$-sparse $\eta$-bounded noise setting (Definition~\ref{def:bn}); all the PAC algorithms above work under isotropic log-concave distributions.}
\label{tab:comp-bn}
\end{table}

\section{Preliminaries}
We consider active learning in the PAC model~\citep{V84, KSS94}.
Denote by $\calX := \R^d$ the instance space, and $\calY := \{-1,+1\}$ the label space.
The learning algorithm is given a data distribution $D$ over $\calX \times \calY$. 
Denote by $D_X$ the marginal distribution of $D$ over $\calX$, and $D_{Y|X}$ the conditional distribution of label given instance.
The learning algorithm is also given a concept class, the set of homogeneous linear classifiers (halfspaces) $\calH:=\{\sign(w \cdot x): w \in \R^d \}$.
For any classifier $h: \calX \to \calY$, we denote by $\err(h):=\P_D(h(x) \neq y)$ the error rate of $h$.
Denote by $h^*$ the optimal classifier in $\calH$: $h^*:=\argmin_{h' \in \calH} \err(h')$.
The excess error of classifier $h$ is defined as $\err(h) - \err(h^*)$; in words, it is
the difference between $h$'s error and the best error in $\calH$. A vector $w$ corresponds to a
linear classifier $h_w := \sign(w \cdot x)$ whose decision boundary has $w$ as its normal; define $w^*$ as the unit vector $w$ such that $h_w = h^*$. We define the angle between two vectors $w, w'$ in $\R^d$ as $\theta(w,w') = \arccos(\frac{w \cdot w'}{\|w\|_2 \| w'\|_2})$. \cite{BL13} shows that there exist numerical constants $C_1, C_2 > 0$, such that if $D_X$ is isotropic log-concave, then for all $w, w'$ in $\R^d$,
\begin{equation}
 C_1 \P_D(h_w(x) \neq h_{w'}(x)) \leq \theta(w,w') \leq C_2 \P_D(h_w(x) \neq h_{w'}(x)).
\label{eqn:angdis}
\end{equation}

In active learning, the algorithm has the ability to draw unlabeled examples from $D_X$ and perform adaptive label queries to a labeling oracle $\calO$.
The oracle $\calO$ takes into input an unlabeled example $x$, and returns a label $y \sim D_{Y|X=x}$.
Given a random variable $z$ whose distribution is $\Delta$ over $\calZ$ and a set $T \subset \calZ$, denote by $\Delta|_T$ the conditional distribution of $z$ given that $z$ is in $T$.
An active learning algorithm is said to $(\epsilon,\delta)$-PAC
learn $\calH$ and $D$ with label complexity $n(\epsilon,\delta)$, if with probability $1-\delta$, it performs at most $n(\epsilon,\delta)$ label queries to $\calO$,
and returns a classifier $\hat{h}$ that has excess error at most $\epsilon$.

Given a vector $w$ and example $(x,y)$, the $\tau$-hinge loss $\ell_{\tau}(w, (x, y))$ is defined as $(1 - \frac{y w \cdot x}{\tau})_+$, where $(z)_+:=\max(0, z)$. Denote by $I(\cdot)$ the indicator function, that is, $I(A)$ is $1$ if predicate $A$ is true, is $0$ if $A$ is false.
A vector $v$ in $\R^d$ is said to be $s$-sparse, if it has at most $s$ nonzero entries.
For an integer $s \in \{1,2,\ldots,d\}$, define $\HT_s(\cdot)$ as the hard thresholding operation that takes a vector $v$ in $\R^d$ as input, and outputs a vector that keeps $v$'s $s$ largest entries in absolute value (breaking ties lexicographically), and setting all its other entries to zero~\citep{BD09}.

In this paper, we focus on the setting where there is a sparse halfspace that performs well under $D$.
Specifically, denote by $\calH_t := \{\sign(w \cdot x): w \in \R^d, \| w \|_0 \leq t \}$ the set of $t$-sparse halfspaces.
We consider the following two conditions on $D$:

\begin{definition}
A distribution $D$ over $\calX \times \calY$ is said to satisfy the {\em $t$-sparse $\nu$-adversarial noise} condition for $\nu \in (0,1)$ and $t \in \{1,\ldots,d\}$, if there is a $t$-sparse unit vector $u$, such that $\P_D(\sign(u \cdot x) \neq y) \leq \nu$.
\label{def:an}
\end{definition}
Observe that under this condition, $h_u$ is not necessarily the optimal classifier in $\calH$; in fact, it may not even be the optimal classifier in $\calH_t$. Nevertheless, by triangle inequality and Equation~\eqref{eqn:angdis}, the angle between $u$ and $w^*$ is at most $O(\nu)$.
It can be readily seen that if $t$ and $\nu$ are larger, the learning problem becomes more difficult. When $t = d$, the condition becomes the $\nu$-adversarial noise condition with respect to $\calH$~\citep{ABL17}.

\begin{definition}
A distribution $D$ over $\calX \times \calY$ is said to satisfy the {\em $t$-sparse $\eta$-bounded noise} condition for $\eta \in [0,\frac 1 2)$ and $t \in \{1,\ldots,d\}$, if there is a $t$-sparse unit vector $u$, such that
for every $x \in \calX$, $\P_D(\sign(u \cdot x) \neq y | x) \leq \eta$.
\label{def:bn}
\end{definition}
Under this condition, it can be seen that $h_u$ is the Bayes optimal classifier, therefore $u$ coincides with $w^*$.
It can be readily seen that if $t$ and $\eta$ are larger, the learning problem becomes more difficult. When $t = d$, the condition becomes the $\eta$-bounded noise condition with respect to $\calH$~\citep{MN06}.

Note that the above two conditions characterize different aspects of the data distribution $D$. The $t$-sparse $\nu$-adversarial noise condition only requires an upper bound on the total label flipping probability. On the other hand, the $t$-sparse $\eta$-bounded noise condition characterizes $D_{Y|X}$ everywhere in $\calX$: for every instance $x$, the expected label $\E[y|x]$ has the same sign as $u \cdot x$.
The following condition is a special case of the above two conditions by setting $\nu = 0$ or $\eta = 0$:

\begin{definition}
A distribution $D$ over $\calX \times \calY$ is said to satisfy the {\em $t$-sparse realizable} condition, for $t \in \{1,2,\ldots,d\}$, if there is a $t$-sparse unit vector $u$, such that $\P_D(\sign(u \cdot x) \neq y) = 0$.
\label{def:r}
\end{definition}

\section{Algorithm}
We present our main algorithm, namely Algorithm~\ref{alg:ae_al} in this section. We defer the exact settings of constants $c_1, c_2, c_3$ to Appendix~\ref{sec:params}.
Our algorithm uses the margin-based active learning framework, initially proposed by~\cite{BBZ07}.
Specifically, it proceeds in epochs, where at each epoch $k$, it draws a sample $S_k$ from distribution $D_X|_{B_k}$, queries their labels, and updates its iterate $w_k$ based on $S_k$. Due to technical reasons, at the first epoch ($k=0$), the sampling region $B_0$ and the constraint set $W_0$ are different from those in subsequent epochs. Throughout the process, the algorithm maintains the invariant that at each epoch $k$, $w_k$ is a  $t$-sparse unit vector.

At each epoch $k \geq 1$, the sampling region $B_k$ is a ``small-margin'' band $\{x: |w_{k-1} \cdot x| \leq b_k \}$, with bandwidth $b_k$ descreasing exponentially in $k$.
Then it performs constrained empirical hinge loss minimization over $S_k$, getting a linear classifier $w_k'$.
The constraint set $W_k$ is the intersection between an $\ell_1$ ball and an $\ell_2$ ball, centered at $w_{k-1}$ with different radii ($\rho_k$ and $r_k$). This is similar to the approach in~\cite{PV13b} for tackling the symmetric noise setting, where a linear optimization problem with a similar shaped constraint set is proposed. The construction of $W_k$'s is inspired by version space constructions in the PAC active learning literature~\citep{CAL94,BBL09,H14}.
Throughout the algorithm, we ensure $W_k$ to satisfy the following two properties with high probability: first, $u$ lie in all the $W_k$'s; second, the $W_k$'s are shrinking in size.\footnote{We refer the reader to Lemma~\ref{lem:induct} for a formal statement.}
In addition, the hinge loss used at epoch $k$ is parameterized by $\tau_k$, which also decreases exponentially in $k$.

Observe that $w_k'$ may not be a sparse vector; therefore, we perform a hard thresholding step (applying $\HT_t$), to ensure that our learned halfspace at the end of round $k$, is $t$-sparse. Hard thresholding has been widely used in the (unquantized) compressed sensing literature~\citep[See e.g.][]{BD09,GK09}, however its utility in one-bit compressed sensing is not yet well-understood. For example, ~\cite{JLBB13} proposes an algorithm named BIHT (binary iterative hard thresholding) that has strong empirical performance, but its convergence properties are unknown.
To the best of our knowledge, our work is the first that establishes convergence guarantees for iterative hard thresholding style algorithms for one-bit compressed sensing. We then perform a $\ell_2$ normalization step to ensure that our iterate $w_k$ is an unit vector, which has a scale comparable to $u$.

Finally, we remark that Algorithm~\ref{alg:ae_al} admits a computationally efficient implementation. First, the sampling regions $B_k$'s can be shown to have probability masses at least $\Omega(\epsilon)$ in $D_X$ for all $k$ in $\{0,1,\ldots,k_0\}$, which makes rejection sampling from $D_X|_{B_k}$ take $O(\frac 1 \epsilon)$ time per example.
Second, optimization problem~\eqref{eqn:opt} is convex, and can be approximately solved by e.g. stochastic gradient descent~\citep[See e.g.][Theorem 2]{SZ13} efficiently.

\begin{algorithm}[t]
  \caption{Attribute and computationally efficient active learning of halfspaces}
\begin{algorithmic}[1]
  \REQUIRE sparsity parameter $t$, target error $\epsilon$, failure probability $\delta$.
  \ENSURE  learned halfspace $\hat{w}$.
  \STATE Initialization: $k_0 \gets \lceil \log_2 \frac {1} {C_1 \epsilon} \rceil$, where $C_1$ is defined in Equation~\eqref{eqn:angdis}.
  \FOR{$k = 0, 1, 2 \ldots,k_0$}
  \STATE $S_k \gets $ sample $n_k = c_1 t (\ln d + \ln \frac{1}{\epsilon} + \ln\frac1{\delta_k})^3$ examples from $D_X|_{B_k}$ and query their labels, where
  \[ B_k := \begin{cases} \R^d, & k = 0, \\ \{x: |w_{k-1} \cdot x| \leq b_k \}, & k \geq 1, \end{cases}\]
	$\delta_k = \frac{\delta}{(k+1)(k+2)}$ and $b_k = c_2 \cdot 2^{-k}$.

  \STATE Solve the following optimization problem:
  \begin{equation}
    w_k' \gets \argmin_{w \in W_k} \sum_{(x,y) \in S_k} \ell_{\tau_k}(w, (x, y)),
    \label{eqn:opt}
  \end{equation}
  where
  \[ W_k = \begin{cases} \{ w \in \R^d: \| w \|_2 \leq 1 \text{ and } \| w \|_1 \leq \sqrt{t} \}, & k = 0, \\ \{ w \in \R^d: \| w - w_{k-1} \|_2 \leq r_k \text{ and } \| w - w_{k-1} \|_1 \leq \rho_k \}, & k \geq 1, \end{cases}
  \]
  $r_k = 2^{-k-3}$, $\rho_k = \sqrt{2t} \cdot 2^{-k-3}$, and $\tau_k = c_3 \cdot 2^{-k}$.
	\label{line:hlm}
  \STATE Let $w_k \gets \frac{\HT_t(w_k')}{\| \HT_t(w_k') \|_2}$.
	\label{line:ht}
  \ENDFOR
  \RETURN $w_{k_0}$.
\end{algorithmic}
\label{alg:ae_al}
\end{algorithm}

\section{Performance guarantees}

In this section, we prove Theorem~\ref{thm:main}, the main result of this paper. 
\begin{theorem}
  There exist numerical constants $\mu_1, \mu_2 \in (0, \frac 1 2)$ such that the following holds.
  Suppose $D_X$ is isotropic log-concave, and one of the following two conditions hold:
  \begin{enumerate}
    \item $D$ satisfies the $t$-sparse $\mu_1\epsilon$-adversarial noise condition;
    \item $D$ satisfies the $t$-sparse $\mu_2$-bounded noise condition.
  \end{enumerate}
	In addition, Algorithm~\ref{alg:ae_al} is run with sparsity parameter $t$, target error $\epsilon$ and failure probability $\delta$.
  Then, with probability $1-\delta$, the output halfspace $\hat{w}$ is such that
  $\err(h_{\hat{w}}) - \err(h^*) \leq \epsilon$,
  and the total number of label queries is $O( t \cdot (\ln d + \ln \frac 1 \epsilon)^3 \cdot \ln \frac 1 \epsilon )$.
  \label{thm:main}
\end{theorem}

As the $t$-sparse realizable setting is a special case of the $t$-sparse adversarial noise setting (by setting $\nu = 0$), Theorem~\ref{thm:main} immediately implies the
following corollary:
\begin{corollary}
	Suppose $D_X$ is isotropic log-concave, and the $t$-sparse realizable condition holds for $D$.
  In addition, Algorithm~\ref{alg:ae_al} is run with sparsity parameter $t$, target error $\epsilon$ and failure probability $\delta$.
  Then, with probability $1-\delta$, the output halfspace $\hat{w}$ is such that
  $\err(h_{\hat{w}}) - \err(h^*) \leq \epsilon$,
  and the total number of label queries is $O( t \cdot (\ln d + \ln \frac 1 \epsilon)^3 \cdot \ln \frac 1 \epsilon )$.
	\label{cor:main}
\end{corollary}

Theorem~\ref{thm:main} and Corollary~\ref{cor:main} imply that, under the respective noise conditions defined above, Algorithm~\ref{alg:ae_al}
has a label complexity of $O( t \polylog(d,  \frac 1 \epsilon) )$. To the best of our knowledge, this
is the first efficient PAC active learning algorithm that has a label complexity linear in $t$, and polylogarithmic
in $d$ and $\frac 1 \epsilon$. Previous works either need to sacrifice computational efficiency to achieve such guarantee~\citep{D05,ZC14}, or have label complexities polynomial in $d$ or $\frac 1 \epsilon$ ~\citep{ABL17, ABHZ16}. We remark that in the membership query model~\citep{A88, BB08}, efficient algorithms with $O( t \polylog(d,  \frac 1 \epsilon) )$ label complexities are implicit in the literature (e.g. by combining \cite{HB11}'s support recovery algorithm with efficient full-dimensional active halfspace learning algorithms~\citep{DKM05,ABL17,CHK17,YZ17}). In contrast, the focus of this paper is on the more challenging PAC setting, and it is unclear how to modify a membership query algorithm to make it work in the PAC setting.

\subsection{Proof of Theorem~\ref{thm:main}}
Recall that $\delta_k = \frac{\delta}{(k+1)(k+2)}$; note that $\sum_{l=0}^{k_0} \delta_k \leq \delta$. To prove Theorem~\ref{thm:main}, we give exact settings of constants $\mu_1, \mu_2 \in (0, \frac 1 2)$ in Appendix~\ref{sec:params},
such that under either the $t$-sparse $\mu_1\epsilon$-adversarial noise condition or the $t$-sparse
$\mu_2$-bounded noise condition, the following lemma holds:
\begin{lemma}
	For every $k \in \{ 0,1,\ldots,k_0 \}$, there is an event $E_k$ with probability $1-\sum_{l=0}^k \delta_l$, on which $u$ is in $W_{k+1}$.
  \label{lem:induct}
\end{lemma}

The proof of Lemma~\ref{lem:induct} relies on the following two supporting lemmas. The first lemma (Lemma~\ref{lem:hlm}) shows that, $w_k'$ produced in the hinge loss minimization step (line~\ref{line:hlm}) has a small angle with $u$. Specifically, the upper bound on $\theta(w_k',u)$
is halved at each iteration $k$, with the help of constrained hinge loss minimization over a fresh set of $n_k = O(t \polylog(d,\frac1\epsilon))$ labeled examples. This relies on two ideas: first, as is standard
in the margin-based active learning framework~\citep[See e.g.][]{BBZ07,BL13}, it suffices to let $w_k'$ achieve a constant error with respect to the sampling distribution at epoch $k$; second, to ensure that the setting of $n_k$ ensures that $w_k'$ indeed has a constant error under the sampling distribution, we use a novel uniform concentration bound of hinge losses of $W_k$ over $S_k$ tighter than all prior works~\citep{ABL17,ABHZ16}. Thanks to our construction of $W_k$, our concentration bound of hinge losses is of order $\tilde{O}(\sqrt{\frac{t\ln d}{n_k}})$, which can be substantially tighter than
$\tilde{O}(\sqrt{\frac{d}{n_k}})$ used in~\citet{ABL17,HKY15} and $\tilde{O}(\sqrt{\frac{(t \ln d) \cdot 2^k}{n_k}})$ used in~\citet{ABHZ16}. We refer the reader to Appendix~\ref{sec:conc} for a formal statement.


\begin{lemma}
For every $k \in \{ 0, 1,\ldots,k_0 \}$, if $u$ is in $W_k$, then with probability $1-\delta_k$, $\theta(w_k', u) \leq 2^{-k-8} \pi$.
\label{lem:hlm}
\end{lemma}

The second lemma (Lemma~\ref{lem:truncate}) shows that, performing a hard thresholding operation followed by $\ell_2$ normalization on $w_k'$ (line~\ref{line:ht}) yields a $t$-sparse unit vector $w_k$ that is close to $u$ in terms of both $\ell_1$ and $\ell_2$ distances. This ensures that $W_{k+1}$, the constraint set of the optimization problem at the next epoch, contains $u$. A key fact used in the proof of the lemma is that, the hard thresholding operator $\HT_t$ is effectively a $\ell_2$-projection onto the $\ell_0$ ball $\{w \in \R^d: \| w \|_0 \leq t \}$.

\begin{lemma}
For every $k \in \{ 0,1,\ldots,k_0 \}$, if $\theta(w_k', u) \leq 2^{-k-8} \pi$, then $u$ is in $W_{k+1}$.
\label{lem:truncate}
\end{lemma}


We are now ready to prove Lemma~\ref{lem:induct}.

\begin{proof}[Proof of Lemma~\ref{lem:induct}]
 We prove the lemma by induction.
\paragraph{Base case.} In the case of $k = 0$, observe that as $u$ has unit $\ell_2$ norm and $u$ is $t$-sparse, by Cauchy-Schwarz, $\| u \|_1 \leq \sqrt{t} \| u\|_2 = \sqrt{t}$. Therefore, $u$ belongs to the set $W_0$ deterministically.
Lemma~\ref{lem:hlm} with $k=0$ shows that there is an event $E_0$ with probability $1-\delta_0$, conditioned on which $\theta(w_0',u) \leq 2^{-8}\pi$. By Lemma~\ref{lem:truncate}, we get that $u$ is in $W_1$.

\paragraph{Inductive case.} For $k \geq 1$, suppose the inductive hypothesis holds. That is, there is an event $E_{k-1}$ with probability $1-\sum_{l=0}^{k-1} \delta_l$, such that
on $E_{k-1}$, $u$ is in $W_k$. By Lemma~\ref{lem:hlm}, there is an event $F_k$ such that $\P(F_k | E_{k-1}) \geq 1 - \delta_k$,
conditioned on which $\theta(w_k', u) \leq 2^{-k-8} \pi$.


Define event $E_k:= E_{k-1} \cap F_k$. Observe that $\P(E_k) = \P(E_{k-1})\P(F_k | E_{k-1}) \geq 1-\sum_{l=0}^{k} \delta_l$.
Now, on event $E_k$, Lemma~\ref{lem:truncate} implies that $u$ is in $W_{k+1}$.
This completes the induction.
\end{proof}

Theorem~\ref{thm:main} is now a direct consequence of Lemma~\ref{lem:induct}; we give its proof below.

\begin{proof}[Proof of Theorem~\ref{thm:main}]
From Lemma~\ref{lem:induct} and the fact that the output $\hat{w}$ is $w_{k_0}$, we have that with probability $1-\sum_{l=0}^{k_0}\delta_l \geq 1-\delta$, $u$ is in $W_{k_0+1}$. By the definition of $W_k$,
\[ \| u - w_{k_0} \|_2 \leq r_{k_0+1} = 2^{-k_0 - 4}. \]
By Lemma~\ref{lem:distangle} in the Appendix and the fact that $\|u\|_2 = 1$, we know that
$\theta(w_{k_0}, u) \leq \pi \| u - w_{k_0} \|_2 \leq 2^{-k_0 - 2} \leq \frac{C_1\epsilon}{2}$.
By the first inequality of Equation~\eqref{eqn:angdis}, we have that
$ \P_D (h_{w_{k_0}}(x) \neq h_{u}(x) ) \leq \frac \epsilon 2. $
Therefore, by triangle inequality and the fact that the output $\hat{w}$ is $w_{k_0}$,
\[ \err(h_{\hat{w}}) - \err(h_u) \leq \frac \epsilon 2. \]
We now consider two separate cases regarding the two different noise conditions:
\begin{enumerate}
\item In the $\mu_1 \epsilon$-adversarial noise setting, we know that $\err(h_u) \leq \mu_1 \epsilon \leq \frac \epsilon 2$.
Therefore,
\[ \err(h_{\hat{w}}) - \err(h^*) \leq \err(h_{\hat{w}}) \leq \err(h_u) + \frac \epsilon 2 \leq \frac \epsilon 2 + \frac \epsilon 2 \leq \epsilon. \]
\item In the $\mu_2$-bounded noise setting, as $h_u$ and $h^*$ are identical,
it immediately follows that $\err(h_{\hat{w}}) - \err(h^*) \leq \frac \epsilon 2 \leq \epsilon$.
\end{enumerate}

We now bound the label complexity of Algorithm~\ref{alg:ae_al}. The total number of labels queried is $\sum_{k=0}^{k_0} n_k$,
where $n_k \leq c_1 \cdot t (\ln d + \ln \frac 1 \epsilon + \ln \frac{k(k+1)}{\delta})^3$, and $k_0 = O(\ln\frac1\epsilon)$.
As a consequence, the total number of label queries is $O(t \cdot (\ln d + \ln \frac 1 \epsilon)^3 \cdot \ln \frac 1 \epsilon)$ in terms of $t, d$ and $\epsilon$.
The theorem follows.
\end{proof}



\section{Conclusions and future work}

We give a computationally efficient PAC active halfspace learning algorithm that enjoys sharp attribute efficient label complexity bounds.
It combines the margin-based framework of~\cite{BBZ07,BL13} with iterative hard thresholding~\citep{BD09, GK09}.
The main novel technical component in our analysis is a uniform concentration bound of hinge losses over shrinking $\ell_1$ balls in the sampling regions.
We outline several promising directions of future research:
\begin{itemize}
\item Can we extend our algorithm to work under $\eta$-bounded noise, when $\eta$ is arbitrarily close to $\frac 1 2$? Recall that the results of \cite{ZC14}
imply a computationally inefficient algorithm with a label complexity of $O(\frac{t \ln d}{(1-2\eta)^2} \ln \frac 1 \epsilon)$ in this setting,
which state of the art computationally efficient algorithms~\citep[e.g.][]{ABHZ16} cannot achieve.

\item Can we design attribute and computationally efficient active learning algorithms that work under broader distributions? Existing results in the active learning and one-bit compressed sensing literature have made substantial progress on settings when the unlabeled distribution is $\alpha$-stable~\citep{L16}, subgaussian~\citep{ALPV14, CB15}, or $s$-concave~\citep{BZ17}; an attribute and computationally efficient, statistically consistent recovery algorithm under any of the above settings would be a step forward.


\item In one-bit compressed sensing, under the symmetric noise condition~\citep{PV13b}, algorithms with sample complexity polynomial in $\frac 1 \epsilon$ have been proposed~\citep{PV13b, ZYJ14, ZG15}.
Can we develop adaptive one-bit compressed sensing algorithms with $O(t \polylog(d,\frac 1 \epsilon))$ measurement complexity in this setting?
\end{itemize}

\section*{Acknowledgments}
I am grateful to Daniel Hsu for suggesting this research direction to me, and many insightful discussions along this line. I would also like to thank Pranjal Awasthi, Jie Shen and Hongyang Zhang for helpful initial conversations about the results in this paper. I thank the anonymous COLT reviewers for their thoughtful comments. Special thanks to Yue Liu, who provided unconditional support throughout this research project.

\bibliography{alsearch}

\begin{thebibliography}{53}
\providecommand{\natexlab}[1]{#1}
\providecommand{\url}[1]{\texttt{#1}}
\expandafter\ifx\csname urlstyle\endcsname\relax
  \providecommand{\doi}[1]{doi: #1}\else
  \providecommand{\doi}{doi: \begingroup \urlstyle{rm}\Url}\fi

\bibitem[Acharya et~al.(2017)Acharya, Bhattacharyya, and Kamath]{ABK17}
Jayadev Acharya, Arnab Bhattacharyya, and Pritish Kamath.
\newblock Improved bounds for universal one-bit compressive sensing.
\newblock In \emph{Information Theory (ISIT), 2017 IEEE International Symposium
  on}, pages 2353--2357. IEEE, 2017.

\bibitem[Ai et~al.(2014)Ai, Lapanowski, Plan, and Vershynin]{ALPV14}
Albert Ai, Alex Lapanowski, Yaniv Plan, and Roman Vershynin.
\newblock One-bit compressed sensing with non-gaussian measurements.
\newblock \emph{Linear Algebra and its Applications}, 441:\penalty0 222--239,
  2014.

\bibitem[Angluin(1988)]{A88}
D.~Angluin.
\newblock Queries and concept learning.
\newblock \emph{Machine Learning}, 2:\penalty0 319--342, 1988.

\bibitem[Awasthi et~al.(2015)Awasthi, Balcan, Haghtalab, and Urner]{ABHU15}
Pranjal Awasthi, Maria-Florina Balcan, Nika Haghtalab, and Ruth Urner.
\newblock Efficient learning of linear separators under bounded noise.
\newblock In Peter Gr{\"{u}}nwald, Elad Hazan, and Satyen Kale, editors,
  \emph{Proceedings of The 28th Conference on Learning Theory, {COLT} 2015,
  Paris, France, July 3-6, 2015}, volume~40 of \emph{{JMLR} Proceedings}, pages
  167--190. JMLR.org, 2015.

\bibitem[Awasthi et~al.(2016)Awasthi, Balcan, Haghtalab, and Zhang]{ABHZ16}
Pranjal Awasthi, Maria-Florina Balcan, Nika Haghtalab, and Hongyang Zhang.
\newblock Learning and 1-bit compressed sensing under asymmetric noise.
\newblock In \emph{Proceedings of The 28th Conference on Learning Theory,
  {COLT} 2016}, 2016.

\bibitem[Awasthi et~al.(2017)Awasthi, Balcan, and Long]{ABL17}
Pranjal Awasthi, Maria~Florina Balcan, and Philip~M Long.
\newblock The power of localization for efficiently learning linear separators
  with noise.
\newblock \emph{Journal of the ACM (JACM)}, 63\penalty0 (6):\penalty0 50, 2017.

\bibitem[Balcan and Long(2013)]{BL13}
M.-F. Balcan and P.~M. Long.
\newblock Active and passive learning of linear separators under log-concave
  distributions.
\newblock In \emph{COLT}, 2013.

\bibitem[Balcan et~al.(2007)Balcan, Broder, and Zhang]{BBZ07}
M.-F. Balcan, A.~Z. Broder, and T.~Zhang.
\newblock Margin based active learning.
\newblock In \emph{COLT}, 2007.

\bibitem[Balcan et~al.(2009)Balcan, Beygelzimer, and Langford]{BBL09}
M.-F. Balcan, A.~Beygelzimer, and J.~Langford.
\newblock Agnostic active learning.
\newblock \emph{J. Comput. Syst. Sci.}, 75\penalty0 (1):\penalty0 78--89, 2009.

\bibitem[Balcan and Zhang(2017)]{BZ17}
Maria-Florina~F Balcan and Hongyang Zhang.
\newblock Sample and computationally efficient learning algorithms under
  s-concave distributions.
\newblock In \emph{Advances in Neural Information Processing Systems}, pages
  4799--4808, 2017.

\bibitem[Bartlett and Mendelson(2002)]{BM02}
Peter~L Bartlett and Shahar Mendelson.
\newblock Rademacher and gaussian complexities: Risk bounds and structural
  results.
\newblock \emph{Journal of Machine Learning Research}, 3\penalty0
  (Nov):\penalty0 463--482, 2002.

\bibitem[Blum(1990)]{B90}
Avrim Blum.
\newblock Learning boolean functions in an infinite attribute space.
\newblock In \emph{Proceedings of the twenty-second annual ACM symposium on
  Theory of computing}, pages 64--72. ACM, 1990.

\bibitem[Blumensath and Davies(2009)]{BD09}
Thomas Blumensath and Mike~E Davies.
\newblock Iterative hard thresholding for compressed sensing.
\newblock \emph{Applied and computational harmonic analysis}, 27\penalty0
  (3):\penalty0 265--274, 2009.

\bibitem[Boufounos and Baraniuk(2008)]{BB08}
Petros~T Boufounos and Richard~G Baraniuk.
\newblock 1-bit compressive sensing.
\newblock In \emph{Information Sciences and Systems, 2008. CISS 2008. 42nd
  Annual Conference on}, pages 16--21. IEEE, 2008.

\bibitem[Candes and Tao(2006)]{CT06}
Emmanuel~J Candes and Terence Tao.
\newblock Near-optimal signal recovery from random projections: Universal
  encoding strategies?
\newblock \emph{IEEE transactions on information theory}, 52\penalty0
  (12):\penalty0 5406--5425, 2006.

\bibitem[Charikar et~al.(2002)Charikar, Chen, and Farach-Colton]{CCF02}
Moses Charikar, Kevin Chen, and Martin Farach-Colton.
\newblock Finding frequent items in data streams.
\newblock In \emph{International Colloquium on Automata, Languages, and
  Programming}, pages 693--703. Springer, 2002.

\bibitem[Chen et~al.(2017)Chen, Hassani, and Karbasi]{CHK17}
Lin Chen, Seyed~Hamed Hassani, and Amin Karbasi.
\newblock Near-optimal active learning of halfspaces via query synthesis in the
  noisy setting.
\newblock In \emph{AAAI}, 2017.

\bibitem[Chen and Banerjee(2015)]{CB15}
Sheng Chen and Arindam Banerjee.
\newblock One-bit compressed sensing with the k-support norm.
\newblock In \emph{Artificial Intelligence and Statistics}, pages 138--146,
  2015.

\bibitem[Cohn et~al.(1994)Cohn, Atlas, and Ladner]{CAL94}
David~A. Cohn, Les~E. Atlas, and Richard~E. Ladner.
\newblock Improving generalization with active learning.
\newblock \emph{Machine Learning}, 15\penalty0 (2):\penalty0 201--221, 1994.

\bibitem[Dasgupta(2005)]{D05}
S.~Dasgupta.
\newblock Coarse sample complexity bounds for active learning.
\newblock In \emph{NIPS}, 2005.

\bibitem[Dasgupta(2011)]{D11}
Sanjoy Dasgupta.
\newblock Two faces of active learning.
\newblock \emph{Theoretical computer science}, 412\penalty0 (19):\penalty0
  1767--1781, 2011.

\bibitem[Dasgupta et~al.(2005)Dasgupta, Kalai, and Monteleoni]{DKM05}
Sanjoy Dasgupta, Adam~Tauman Kalai, and Claire Monteleoni.
\newblock Analysis of perceptron-based active learning.
\newblock In \emph{Learning Theory, 18th Annual Conference on Learning Theory,
  {COLT} 2005, Bertinoro, Italy, June 27-30, 2005, Proceedings}, pages
  249--263, 2005.

\bibitem[Donoho(2006)]{D06}
David~L Donoho.
\newblock Compressed sensing.
\newblock \emph{IEEE Transactions on information theory}, 52\penalty0
  (4):\penalty0 1289--1306, 2006.

\bibitem[Feldman(2007)]{F07}
Vitaly Feldman.
\newblock Attribute-efficient and non-adaptive learning of parities and dnf
  expressions.
\newblock \emph{Journal of Machine Learning Research}, 8\penalty0
  (Jul):\penalty0 1431--1460, 2007.

\bibitem[Garg and Khandekar(2009)]{GK09}
Rahul Garg and Rohit Khandekar.
\newblock Gradient descent with sparsification: an iterative algorithm for
  sparse recovery with restricted isometry property.
\newblock In \emph{Proceedings of the 26th Annual International Conference on
  Machine Learning}, pages 337--344. ACM, 2009.

\bibitem[Gopi et~al.(2013)Gopi, Netrapalli, Jain, and Nori]{GNJN13}
Sivakant Gopi, Praneeth Netrapalli, Prateek Jain, and Aditya Nori.
\newblock One-bit compressed sensing: Provable support and vector recovery.
\newblock In \emph{International Conference on Machine Learning}, pages
  154--162, 2013.

\bibitem[Gupta et~al.(2010)Gupta, Nowak, and Recht]{GNR10}
Ankit Gupta, Robert Nowak, and Benjamin Recht.
\newblock Sample complexity for 1-bit compressed sensing and sparse
  classification.
\newblock In \emph{Information Theory Proceedings (ISIT), 2010 IEEE
  International Symposium on}, pages 1553--1557. IEEE, 2010.

\bibitem[Hanneke(2007)]{H07}
S.~Hanneke.
\newblock A bound on the label complexity of agnostic active learning.
\newblock In \emph{ICML}, 2007.

\bibitem[Hanneke(2014)]{H14}
Steve Hanneke.
\newblock Theory of disagreement-based active learning.
\newblock \emph{Foundations and Trends{\textregistered} in Machine Learning},
  7\penalty0 (2-3):\penalty0 131--309, 2014.

\bibitem[Hanneke et~al.(2015)Hanneke, Kanade, and Yang]{HKY15}
Steve Hanneke, Varun Kanade, and Liu Yang.
\newblock Learning with a drifting target concept.
\newblock In \emph{International Conference on Algorithmic Learning Theory},
  pages 149--164. Springer, 2015.

\bibitem[Haupt and Baraniuk(2011)]{HB11}
Jarvis Haupt and Richard Baraniuk.
\newblock Robust support recovery using sparse compressive sensing matrices.
\newblock In \emph{Information Sciences and Systems (CISS), 2011 45th Annual
  Conference on}, pages 1--6. IEEE, 2011.

\bibitem[Jacques et~al.(2013)Jacques, Laska, Boufounos, and Baraniuk]{JLBB13}
Laurent Jacques, Jason~N Laska, Petros~T Boufounos, and Richard~G Baraniuk.
\newblock Robust 1-bit compressive sensing via binary stable embeddings of
  sparse vectors.
\newblock \emph{IEEE Transactions on Information Theory}, 59\penalty0
  (4):\penalty0 2082--2102, 2013.

\bibitem[Kakade et~al.(2009)Kakade, Sridharan, and Tewari]{KST09}
Sham~M Kakade, Karthik Sridharan, and Ambuj Tewari.
\newblock On the complexity of linear prediction: Risk bounds, margin bounds,
  and regularization.
\newblock In \emph{Advances in neural information processing systems}, pages
  793--800, 2009.

\bibitem[Kearns et~al.(1994)Kearns, Schapire, and Sellie]{KSS94}
Michael~J Kearns, Robert~E Schapire, and Linda~M Sellie.
\newblock Toward efficient agnostic learning.
\newblock \emph{Machine Learning}, 17\penalty0 (2-3):\penalty0 115--141, 1994.

\bibitem[Klivans and Servedio(2006)]{KS06}
Adam~R Klivans and Rocco~A Servedio.
\newblock Toward attribute efficient learning of decision lists and parities.
\newblock \emph{Journal of Machine Learning Research}, 7\penalty0
  (Apr):\penalty0 587--602, 2006.

\bibitem[Kulkarni et~al.(1993)Kulkarni, Mitter, and Tsitsiklis]{KMT93}
Sanjeev~R Kulkarni, Sanjoy~K Mitter, and John~N Tsitsiklis.
\newblock Active learning using arbitrary binary valued queries.
\newblock \emph{Machine Learning}, 11\penalty0 (1):\penalty0 23--35, 1993.

\bibitem[Li(2016)]{L16}
Ping Li.
\newblock One scan 1-bit compressed sensing.
\newblock In \emph{Artificial Intelligence and Statistics}, pages 1515--1523,
  2016.

\bibitem[Littlestone(1987)]{L87}
Nick Littlestone.
\newblock Learning quickly when irrelevant attributes abound: {A} new
  linear-threshold algorithm.
\newblock \emph{Machine Learning}, 2\penalty0 (4):\penalty0 285--318, 1987.

\bibitem[Long and Servedio(2007)]{LS07}
Philip~M Long and Rocco Servedio.
\newblock Attribute-efficient learning of decision lists and linear threshold
  functions under unconcentrated distributions.
\newblock In \emph{Advances in Neural Information Processing Systems}, pages
  921--928, 2007.

\bibitem[Lov{\'a}sz and Vempala(2007)]{LV07}
L{\'a}szl{\'o} Lov{\'a}sz and Santosh Vempala.
\newblock The geometry of logconcave functions and sampling algorithms.
\newblock \emph{Random Structures \& Algorithms}, 30\penalty0 (3):\penalty0
  307--358, 2007.

\bibitem[Massart and N{\'e}d{\'e}lec(2006)]{MN06}
Pascal Massart and {\'E}lodie N{\'e}d{\'e}lec.
\newblock Risk bounds for statistical learning.
\newblock \emph{The Annals of Statistics}, pages 2326--2366, 2006.

\bibitem[Natarajan(1995)]{N95}
Balas~Kausik Natarajan.
\newblock Sparse approximate solutions to linear systems.
\newblock \emph{SIAM journal on computing}, 24\penalty0 (2):\penalty0 227--234,
  1995.

\bibitem[Plan and Vershynin(2013{\natexlab{a}})]{PV13a}
Yaniv Plan and Roman Vershynin.
\newblock One-bit compressed sensing by linear programming.
\newblock \emph{Communications on Pure and Applied Mathematics}, 66\penalty0
  (8):\penalty0 1275--1297, 2013{\natexlab{a}}.

\bibitem[Plan and Vershynin(2013{\natexlab{b}})]{PV13b}
Yaniv Plan and Roman Vershynin.
\newblock Robust 1-bit compressed sensing and sparse logistic regression: A
  convex programming approach.
\newblock \emph{IEEE Transactions on Information Theory}, 59\penalty0
  (1):\penalty0 482--494, 2013{\natexlab{b}}.

\bibitem[Servedio et~al.(2012)Servedio, Tan, and Thaler]{STT12}
Rocco Servedio, Li-Yang Tan, and Justin Thaler.
\newblock Attribute-efficient learning andweight-degree tradeoffs for
  polynomial threshold functions.
\newblock In \emph{Conference on Learning Theory}, pages 14--1, 2012.

\bibitem[Servedio(2000)]{S00}
Rocco~A Servedio.
\newblock Computational sample complexity and attribute-efficient learning.
\newblock \emph{Journal of Computer and System Sciences}, 60\penalty0
  (1):\penalty0 161--178, 2000.

\bibitem[Settles(2010)]{S10}
Burr Settles.
\newblock Active learning literature survey.
\newblock \emph{University of Wisconsin, Madison}, 52\penalty0
  (55-66):\penalty0 11, 2010.

\bibitem[Shamir and Zhang(2013)]{SZ13}
Ohad Shamir and Tong Zhang.
\newblock Stochastic gradient descent for non-smooth optimization: Convergence
  results and optimal averaging schemes.
\newblock In \emph{International Conference on Machine Learning}, pages 71--79,
  2013.

\bibitem[Valiant(1984)]{V84}
Leslie~G Valiant.
\newblock A theory of the learnable.
\newblock \emph{Communications of the ACM}, 27\penalty0 (11):\penalty0
  1134--1142, 1984.

\bibitem[Yan and Zhang(2017)]{YZ17}
Songbai Yan and Chicheng Zhang.
\newblock Revisiting perceptron: Efficient and label-optimal learning of
  halfspaces.
\newblock In \emph{Advances in Neural Information Processing Systems}, pages
  1056--1066, 2017.

\bibitem[Zhang and Chaudhuri(2014)]{ZC14}
Chicheng Zhang and Kamalika Chaudhuri.
\newblock Beyond disagreement-based agnostic active learning.
\newblock In \emph{Advances in Neural Information Processing Systems 27: Annual
  Conference on Neural Information Processing Systems 2014, December 8-13 2014,
  Montreal, Quebec, Canada}, pages 442--450, 2014.

\bibitem[Zhang et~al.(2014)Zhang, Yi, and Jin]{ZYJ14}
Lijun Zhang, Jinfeng Yi, and Rong Jin.
\newblock Efficient algorithms for robust one-bit compressive sensing.
\newblock In \emph{International Conference on Machine Learning}, pages
  820--828, 2014.

\bibitem[Zhu and Gu(2015)]{ZG15}
Rongda Zhu and Quanquan Gu.
\newblock Towards a lower sample complexity for robust one-bit compressed
  sensing.
\newblock In \emph{International Conference on Machine Learning}, pages
  739--747, 2015.

\end{thebibliography}

\newpage
\appendix
\section{Detailed choices of learning and problem parameters}
\label{sec:params}

In this section, we give the exact settings of $c_1, c_2, c_3$ that appears in Algorithm~\ref{alg:ae_al}, and $\mu_1, \mu_2$, the noise rates
that can be tolerated by Algorithm~\ref{alg:ae_al} under the two noise conditions.

Define $D_k$ as
the distribution $D$ over $(x,y)$ conditioned on that $x$ lies in $B_k$.
Although cannot be sampled from directly, for analysis purposes, we define $\tilde{D}$ as the joint distribution of $(x, \sign(u \cdot x))$, and
$\tilde{D}_k$ as the distribution of $\tilde{D}$ conditioned on that $x$ lies in $B_k$.
Let $\lambda > 0$ be a constant, which will be specified at the end of this section.
Given $\lambda$, we define $c_2:=c_2(\lambda)$ such that:
\begin{enumerate}
\item $c_2(\lambda) = O(\ln \frac{1}{\lambda})$,
\item For all $w$ such that $\theta(w,w_{k-1}) \leq 2^{-k-3} \pi$,
\begin{equation}
  \P_D( \sign(w \cdot x) \neq \sign(w_{k-1} \cdot x), |w_{k-1} \cdot x| \geq c_2(\lambda) \cdot 2^{-k} ) \leq \lambda \cdot 2^{-k}.
\label{eqn:margin}
\end{equation}
\end{enumerate}
The existence of such function $c_2(\cdot)$ is guaranteed by Theorem 21 of \cite{BL13}, along with the fact that $D_X$ is isotropic log-concave.

In addition, given $\lambda > 0$, define $c_3(\lambda) := \lambda \min(\cf /81, \cf  c_2(\lambda)/9)$ (where $\cf $ is a numerical constant defined in Lemma~\ref{lem:bandmass}), such that $\tau_k = c_3 2^{-k}$. Under this setting of $\tau_k$, we have that for all $k$ in $\{0,1,\ldots,k_0\}$:
\begin{equation}
\E_{D_k} \ell_{\tau_k}(u, x, y) \leq \P_{D_k}(|u \cdot x| \leq \tau_k) \leq \frac{\P_{D_X}(| u \cdot x | \leq \tau_k)}{\P_{D_X}( x \in B_k )} \leq \frac{9 \tau_k}{\min(\cf /9, \cf  b_k)} \leq \lambda.
\label{eqn:lowerr}
\end{equation}
where the first inequality is from that $\ell_{\tau_k}(u, (x, \sign(u \cdot x))) \in [0,1]$, and $\ell_{\tau_k}(u, (x, \sign(u \cdot x))) = 0$ if $|u \cdot x| \geq \tau_k$;
the second inequality uses the fact that $\P(A|B) \leq \frac{\P(A)}{\P(B)}$ for any two events $A,B$; the third inequality uses Lemma~\ref{lem:bandmass} to upper bound (resp. lower bound) the numerator (resp. the denominator).

Recall that $n_k := c_1 t (\ln d + \ln \frac 1 \epsilon + \ln \frac{1}{\delta_k})^3$. Given $\lambda > 0$ and $c_2(\lambda)$, $c_3(\lambda)$, we set $c_1:=c_1(\lambda)$ such that by Lemmas~\ref{lem:conc}, for all $k$ in $\{0,1,\ldots,k_0\}$, for all $w$ in $W_k$,
\begin{equation}
 |\E_{S_k} \ell_{\tau_k}(w, (x, y)) - \E_{D_k} \ell_{\tau_k}(w, (x, y))| \leq \lambda.
\label{eqn:conc}
\end{equation}

Given $\lambda$ and $c_2(\lambda)$, $c_3(\lambda)$, we also choose $\mu_1 = \mu_1(\lambda), \mu_2 = \mu_2(\lambda) \in (0,\frac 1 2)$ such that under the respective noise condition, for all $k$ in $\{0,1,\ldots,k_0\}$, for all $w$ in $W_k$,
\begin{equation}
| \E_{\tilde{D}_k} \ell_{\tau_k}(w, (x, y)) - \E_{D_k} \ell_{\tau_k}(w, (x, y))| \leq \lambda.
\label{eqn:corrupt}
\end{equation}
The existences of $\mu_1(\lambda)$ and $\mu_2(\lambda)$ are guaranteed in light of Lemmas~\ref{lem:tv-an} and~\ref{lem:tv-bn}.

Define $f(\lambda') = C_2(45 c_2(\lambda') \lambda' + 5\lambda')$. Observe that by the definition of $c_2(\cdot)$, $f(\lambda')$ goes to zero as $\lambda'$ goes down to zero. Therefore, we can select a value of $\lambda > 0$, such that $f(\lambda) \leq 2^{-8} \pi$.
Note that our selection of $\lambda$ also determines the value of $c_1$, $c_2$, $c_3$ and $\mu_1$, $\mu_2$.

\section{Learning guarantee at each epoch}
\label{sec:epoch}
In this section, we prove two key lemmas, namely
Lemmas~\ref{lem:hlm} and~\ref{lem:truncate}, both of which serve as the basis for Lemma~\ref{lem:induct}.

\subsection{Proof of Lemma~\ref{lem:hlm}}
The proof of Lemma~\ref{lem:hlm} is based on
a uniform concentration bound on the $\tau_k$-hinge loss over $W_k$ in the sampling region $B_k$, namely Lemma~\ref{lem:conc}.
Specifically, Lemma~\ref{lem:conc} implies that the
difference between the empirical hinge losses and the expected hinge losses for all $w$ in $W_k$ with respect to $D_k$ is uniformly bounded by $\tilde{O}(\sqrt{\frac{t (\ln d + \ln \frac 1 \epsilon)^3}{n_k}})$.
As will be seen in the analysis, only a constant concentration error $\lambda$ is required in the hinge loss minimization step (see Equation~\eqref{eqn:conc}). Therefore, the setting of $n_k = O(t (\ln d + \ln \frac 1 \epsilon)^3)$ fulfills this requirement.

\begin{proof}[Proof of Lemma~\ref{lem:hlm}]
We consider the cases of $k = 0$ and $k \geq 1$ separately.
\paragraph{Case 1: $k = 0$.} By Lemma~\ref{lem:hlm-grt} below and the fact that $D = D_0$, $\P_D(\sign(w_0' \cdot x) \neq \sign(u \cdot x)) \leq 5\lambda$ holds.
In addition, by the second inequality of of Equation~\eqref{eqn:angdis},
we have that $\theta(w_0',u) \leq 5 C_2 \lambda$.
By the definiton of $\lambda$, it is at most $2^{-8} \pi$.

\paragraph{Case 2: $k \geq 1$.} By Lemma~\ref{lem:hlm-grt} below, $\P_{D_k}(\sign(w_k' \cdot x) \neq \sign(u \cdot x)) \leq 5\lambda$ holds. We now show that the above fact implies that the angle between $w_k'$ and $u$ is
at most $2^{-k-8}\pi$. This implication is well known in the margin-based
active learning literature~\citep{BBZ07, BL13}; we provide the proof here for completeness.

By Lemma~\ref{lem:bandmass}, $\P_{D_k}(\sign(w_k' \cdot x) \neq \sign(u \cdot x)) \leq 5\lambda$ implies that
 \begin{eqnarray}
	 &&\P_{D}(\sign(w_k' \cdot x) \neq \sign(u \cdot x), x \in B_k) \nonumber \\
	 &=& \P_{D_k}(\sign(w_k' \cdot x) \neq \sign(u \cdot x)) \cdot \P_{D_X}(x \in B_k) \leq 5 \lambda \cdot 9 c_2(\lambda) 2^{-k} \leq 45 \lambda c_2(\lambda) 2^{-k}.
	 \label{eqn:inband}
\end{eqnarray}

On the other hand, observe that for all $w$ in $W_k$, $\| w - w_{k-1} \|_2 \leq 2^{-k-3}$.
Using Lemma~\ref{lem:distangle} and the fact that $w_{k-1}$ is a unit vector, we get that for all $w$ in $W_k$,
$ \theta(w,w_{k-1}) \leq 2^{-k-3}\pi. $
Specifically, by Equation~\eqref{eqn:margin}, we have that
	\[ \P_D(\sign(w \cdot x) \neq \sign(w_{k-1} \cdot x), x \notin B_k) \leq \lambda 2^{-k} \]
 holds for $w \in \{u, w_k'\} \subset W_k$ respectively. Therefore, by triangle inequality,
 \begin{eqnarray}
&& \P_D(\sign(w_k' \cdot x) \neq \sign(u \cdot x), x \notin B_k) \nonumber \\
&\leq& \P_D(\sign(w_k' \cdot x) \neq \sign(w_{k-1} \cdot x), x \notin B_k) +  \P_D(\sign(u \cdot x) \neq \sign(w_{k-1} \cdot x), x \notin B_k) \nonumber \\
&\leq& 2 \lambda 2^{-k}. \label{eqn:outband}
\end{eqnarray}



Combining Equations~\eqref{eqn:inband} and~\eqref{eqn:outband}, we have that
	\[ \P_D(\sign(w_k' \cdot x) \neq \sign(u \cdot x) ) \leq (45 c_2(\lambda) \lambda + 2 \lambda) 2^{-k}. \]

Applying the second inequality of Equation~\eqref{eqn:angdis} gives that
	\[ \theta(w_k', u) \leq C_2 (45 c_2(\lambda) + 2) \lambda 2^{-k}. \]

By the definition of $\lambda$, the above is at most $2^{-k-8} \pi$. 

Combining the above two cases, the lemma follows.
\end{proof}

\begin{lemma}
For every $k$ in $\{0,1,\ldots,k_0\}$, if $u$ is in $W_k$, then
\[ \P_{D_k}(\sign(w_k' \cdot x) \neq \sign(u \cdot x)) \leq 5\lambda. \]
\label{lem:hlm-grt}
\end{lemma}
\begin{proof}
	If $u$ is in $W_k$, then we have the following chain of inequalities:
	\begin{eqnarray*}
	  \P_{D_k}(\sign(w_k' \cdot x) \neq \sign(u \cdot x))
		&=& \P_{\tilde{D}_k}(\sign(w_k' \cdot x) \neq y) \\
	  &\leq& \E_{\tilde{D}_k} \ell_{\tau_k}(w_k', (x, y)) \\
	  &\leq& \E_{D_k} \ell_{\tau_k}(w_k', (x, y)) + \lambda  \\
	  &\leq& \E_{S_k} \ell_{\tau_k}(w_k', (x, y)) + 2\lambda \\
	  &\leq& \E_{S_k} \ell_{\tau_k}(u, (x, y)) + 2\lambda \\
	  &\leq& \E_{D_k} \ell_{\tau_k}(u, (x, y)) + 3\lambda \\
	  &\leq& \E_{\tilde{D}_k} \ell_{\tau_k}(u, (x, y)) + 4\lambda \\
	  &\leq& \lambda + 4\lambda = 5\lambda, \\
	\end{eqnarray*}
	where the first inequality is from the fact that the $\tau_k$-hinge loss is an upper bound of the 0-1 loss; the second inequality is from
	Equation~\eqref{eqn:corrupt} and that $w_k' \in W_k$; the third inequality is from Equation~\eqref{eqn:conc} and that $w_k' \in W_k$;
	the fourth inequality is by the optimality of $w_k'$ in optimization problem~\eqref{eqn:opt} and that $u \in W_k$; the fifth inequality is from Equation~\eqref{eqn:conc} and that $u \in W_k$;
	the sixth inequality is from Equation~\eqref{eqn:corrupt} and that $u \in W_k$; the last inequality is from Equation~\eqref{eqn:lowerr}.
\end{proof}

The following lemma is used in the proof of Lemma~\ref{lem:hlm}; it establishes a connection between the angle and the $\ell_2$ distance of two vectors, when one of the vectors has unit $\ell_2$ norm.

\begin{lemma}
Suppose $v$ is an unit vector in $\R^d$ (that is, $\|v\|_2 = 1$). Then, for any vector $w$ in $\R^d$,
$\theta(w,v) \leq \pi \| w - v \|_2$.
\label{lem:distangle}
\end{lemma}
\begin{proof}
Denote by $\hat{w}$ the $\ell_2$ normalized version of $w$, i.e. $\hat{w} = \frac{w}{\|w\|_2}$.
Lemma~\ref{lem:normalize} below implies that
\begin{equation}
	\| \hat{w} - v \|_2 \leq 2\| w - v \|_2.
	\label{eqn:normalizedl2}
\end{equation}
Consequently,
\[ \theta(w, v) \leq \frac{\pi}{2} \cdot 2\sin\frac{\theta(w, v)}{2} = \frac{\pi}{2} \cdot 2\sin\frac{\theta(\hat{w}, v)}{2} = \frac{\pi}{2} \| \hat{w} - v \|_2  \leq \pi \| w - v \|_2.\]
where the first inequality is from the elementary inequality that $\phi \leq \frac \pi 2 \sin \phi$ for $\phi \in [0,\frac \pi 2]$ (by taking $\phi = \frac{\theta(w, v)}{2}$),
the second inequality is from the identity that $\| \hat{w} - v\|_2 = 2\sin\frac{\theta(\hat{w}, v)}2$ as both $\hat{w}$ and $v$ are unit vectors, and the last inequality is
from Equation~\eqref{eqn:normalizedl2}.
\end{proof}

The following lemma is used in the proof of Lemma~\ref{lem:distangle}; it uses the fact that $\ell_2$ normalization is an $\ell_2$ projection onto the unit sphere.

\begin{lemma}
Suppose $v$ is an unit vector in $\R^d$ (that is, $\|v\|_2 = 1$). Then, for any vector $w$ in $\R^d$,
\[ \| \frac{w}{\|w\|_2} - v \|_2 \leq 2 \| w - v \|_2. \]
\label{lem:normalize}
\end{lemma}
\begin{proof}
Denote by $\hat{w}$ the $\ell_2$ normalized version of $w$, i.e. $\hat{w} = \frac{w}{\|w\|_2}$.
We have that by triangle inequality,
\[ \| \hat{w} - w \|_2 = \| (\frac{1}{\|w\|_2}-1) w \|_2 = | \| w \|_2  - 1 | = | \| w \|_2  - \| v \|_2 | \leq \| w - v \|_2. \]
Again by triangle inequality,
\[ \| \hat{w} - v \|_2 \leq \| \hat{w} - w \|_2 + \| w - v \|_2 \leq 2\| w - v \|_2. \]
The lemma follows.
\end{proof}


\subsection{Proof of Lemma~\ref{lem:truncate}}

The proof of Lemma~\ref{lem:truncate} is based on the key insight that the hard thresholding operation $\HT_t$
 is effectively a projection onto the $\ell_0$-ball $\{w \in \R^d: \| w \|_0 \leq t\}$; see Lemma~\ref{lem:ht}
 for a formal description.

\begin{proof}[Proof of Lemma~\ref{lem:truncate}]
  Denote by $\hat{w}_k'$ the $\ell_2$ normalized version of $w_k'$: $\hat{w}_k':=\frac{w_k'}{\| w_k' \|_2}$.
  Under the condition that $\theta(w_k', u) \leq 2^{-k-8} \pi$, as $\hat{w}_k'$ and $u$ are both unit vectors, we have
\begin{equation*}
  \| \hat{w}_k' - u \|_2 = 2 \sin \frac{\theta(w_k', u)}{2} \leq \theta(w_k', u) \leq 2^{-k-8} \pi \leq 2^{-k-6}.
\end{equation*}
Now, by Lemma~\ref{lem:ht} below, we have that
$\| \hat{w}_k' - \HT_t(\hat{w}_k') \|_2 \leq   \| \hat{w}_k' - u \|_2 \leq 2^{-k-6}$. By triangle inequality of $\ell_2$ distance, we have that
\begin{equation*}
 \| \HT_t(\hat{w}_k') - u \|_2 \leq \| \hat{w}_k' - w_k \|_2 + \| \hat{w}_k' - u \|_2 \leq 2^{-k-5}.
\end{equation*}
Observe that as
$w_k$ and $\hat{w}_k'$ are equal up to scaling, $w_k := \frac{\HT_t(w_k')}{\| \HT_t(w_k') \|_2}$ is identically $\frac{\HT_t(\hat{w}_k')}{\| \HT_t(\hat{w}_k') \|_2}$.
Applying Lemma~\ref{lem:normalize} with $w = \HT_t(\hat{w}_k')$ and $v = u$, we get that
\[ \| w_k - u \|_2 \leq 2 \| \HT_t(\hat{w}_k') - u \|_2 \leq 2^{-k-4} = r_{k+1}. \]
In addition, as $w_k$ and $u$ are both $t$-sparse, $w_k - u$ is $2t$-sparse. Therefore, by Cauchy-Schwarz, $ \| w_k - u \|_1 \leq \sqrt{2t} \| w_k - u \|_2 \leq \sqrt{2t} r_{k+1} = \rho_{k+1}$.
Hence, $u$ is in the set $\{ w \in \R^d: \| w - w_k \|_2 \leq r_{k+1} \text{ and } \| w - w_k \|_1 \leq \rho_{k+1} \}$, namely $W_{k+1}$.
\end{proof}




\begin{lemma}
Suppose $w$ is a vector in $\R^d$. Then, for any $t$-sparse vector $v$ in $\R^d$,
\[ \| \HT_t(w) - w \|_2 \leq \| v - w \|_2. \]
In other words, $\HT_t(w)$ is the best $t$-sparse approximation to $w$, measured in $\ell_2$ distance.
\label{lem:ht}
\end{lemma}
\begin{proof}
Denote by $w_{(1)}, w_{(2)}, \ldots, w_{(d)}$ the $d$ entries of $w$ in descending
order in magnitude. We have that
\[ \| \HT_t(w) - w \|_2^2 = \sum_{i=t+1}^d w_{(i)}^2. \]
On the other hand, for any $t$-sparse vector $v$, denote by $S$ its support ($|S| \leq t$). We have that
\[ \| v - w \|_2^2 \geq \sum_{i \in \{1,\ldots,d\} \setminus S} w_i^2 \geq \sum_{i=t+1}^d w_{(i)}^2, \]
where the second inequality is from that the sum of squares of any $d-t$ entries in $w$ must be greater than that of the bottom $d-t$ entries.
The lemma follows.
\end{proof}



\section{The uniform concentration of hinge losses in label query regions}
\label{sec:conc}
In contrast to~\cite{ABHZ16} where the constraint set of the hinge loss minimization problem at epoch $k$ is the intersection of an $\ell_2$ ball of $O(2^{-k})$ radius and an $\ell_1$ ball of $O(\sqrt{t})$ radius,
Algorithm~\ref{alg:ae_al} defines the constraint set $W_k$ to be the intersection of an $\ell_2$ ball of $O(2^{-k})$ radius and an $\ell_1$ ball of $O(\sqrt{t} 2^{-k})$ radius. The following key lemma, namely Lemma~\ref{lem:conc}, shows the advantage of our construction of $W_k$.
Specifically, it establishes a sharp uniform concentration of hinge losses $\ell_{\tau_k}$ over $W_k$, with respect to sample $S_k$ drawn from distribution $D_k$. Observe that the concentration bound is $\tilde{O}(\sqrt{\frac{t (\ln d + \ln \frac 1 \epsilon)^3}{n_k}})$; if one were to use the constraint set in~\cite{ABHZ16},
one would get concentration bounds of order $\tilde{O}(\sqrt{\frac{ (t \ln d) \cdot 2^k}{n_k}})$, which has an exponential dependence in $k$.

\begin{lemma}
For any $c_2, c_3 > 0$, there exists a constant $\cs  > 0$ such that the following holds.
Given $k$ in $\{0, 1,\ldots,k_0\}$,
suppose $S_k$ is a sample of size $n_k$ drawn from distribution $D_k$. Then with probability
$1-\delta_k$, for all $w \in W_k$, we have:
\[
|\E_{S_k} \ell_{\tau_k}(w, (x, y)) - \E_{D_k} \ell_{\tau_k}(w, (x, y))|
\leq \cs  \ln\frac{n_k d}{\epsilon \delta_k} \cdot \sqrt{\frac{t(\ln d + \ln \frac 2 {\delta_k})}{n_k}}.
\]
\label{lem:conc}
\end{lemma}

Before going into the proof of the lemma, let us define some notations. For every $k$ in $\{0,1,\ldots,k_0\}$, denote by $R_k = \cse \ln(\frac{2n_k d}{\delta_k} \max(\frac9{\cf }, \frac{1}{\cf b_k}))$ for some large enough positive constant $\cse $ such that
$\P_{D_X} (\| x \|_\infty > R_k) \leq \min(\cf /9, \cf  b_k) \delta_k / 2n_k$ holds.
The existence of such $\cse$ is guaranteed by Lemma 20 of~\citet{ABHZ16}.
In addition, define $T_k:= \{(x,y): \| x \|_\infty \leq R_k \}$.

The proof of Lemma~\ref{lem:conc} relies on the following observation: as the marginal distribution of $D_k$ over $\calX$ has a light tail, the probability that $(x,y) \notin T_k$ is extremely small, therefore $D_k|_{T_k}$ is ``close'' to $D_k$.
The subsequent reasoning is composed of two parts: first, we show that $|\E_{S_k} \ell_{\tau_k}(w, (x, y)) - \E_{D_k |_{T_k}} \ell_{\tau_k}(w, (x, y))|$ is small (Lemma~\ref{lem:sktk}). To this end, we argue that $S_k$ is almost a sample iid from $D_k |_{T_k}$, and then carefully apply Rademacher complexity
bounds for $\ell_1$ bounded linear predictors on $\ell_\infty$ bounded examples~\citep{KST09}. Second, we show that $|\E_{D_k |_{T_k}} \ell_{\tau_k}(w, (x, y)) - \E_{D_k} \ell_{\tau_k}(w, (x, y)) |$ is small for all $w$ in $W_k$ (Lemma~\ref{lem:tk}).

\begin{proof}
First we show that there is an event $E$ that has probability at least $1-\delta_k/2$,
conditioned on which all the unlabeled examples in $S_k$ have $\ell_\infty$ norms uniformly bounded by $R_k$.
Define:
\begin{equation}
	E := \{ \text{ for all } (x,y) \text{ in } S_k, (x,y) \text{ is in } T_k \}.
	\label{eqn:e}
\end{equation}
Observe that for each individual $(x,y)$ in $S_k$ drawn from $D_k$,
\[ \P_{D_k} ((x,y) \notin T_k) \leq \frac{\P_{D_k} ((x,y) \notin T_k)}{\P_{D_X}(x \in B_k)} \leq \frac{\min(\cf /9, \cf  b_k) \delta_k / 2n_k}{\min(\cf /9, \cf  b_k)} \leq \frac{\delta_k}{2n_k}, \]
therefore, by union bound,
$\P(E) \geq 1-\delta_k/2$.

By Lemma~\ref{lem:sktk}, there is an event $F$ such that
$\P[F|E] \geq 1-\delta_k/2$, and on event $F$,
\begin{eqnarray}
\left| \E_{S_k} \ell_{\tau_k}(w, (x, y)) - \E_{D_k |_{T_k}} \ell_{\tau_k}(w, (x, y)) \right|
&\leq& \cfo  \cdot \ln\frac{n_k d}{\epsilon \delta_k} \cdot \sqrt{\frac{t(2\ln d + \ln \frac 2 {\delta_k})}{n_k}},
\label{eqn:sktk}
\end{eqnarray}
for some constant $\cfo $ defined in Lemma~\ref{lem:sktk}.

Note that $\P(E \cap F) \geq (1 - \delta_k/2)^2 \geq 1-\delta_k$. We henceforth condition on $E \cap F$ happening.

Using Lemma~\ref{lem:tk}, we get that for all $w$ in $W_k$,
\begin{equation}
|\E_{D_k|_{T_k}} \ell_{\tau_k}(w, (x, y)) - \E_{D_k} \ell_{\tau_k}(w, (x, y))| \leq C_9 \sqrt{\frac{1}{n_k}},
\label{eqn:tk}
\end{equation}
for some constant $\cn $ defined in Lemma~\ref{lem:tk}.



Combining Equations~\eqref{eqn:sktk} and~\eqref{eqn:tk}, we conclude that there is a constant $\cs $ such that on event $E \cap F$,
\[
|\E_{S_k} \ell_{\tau_k}(w, (x, y)) - \E_{D_k} \ell_{\tau_k}(w, (x, y))|
\leq \cs  \ln\frac{n_k d}{\epsilon \delta_k} \cdot \sqrt{\frac{t(\ln d + \ln \frac 2 {\delta_k})}{n_k}}.
\]
This proves the lemma.
\end{proof}

\begin{lemma}
For every $k$ in $\{0,1,\ldots,k_0\}$, suppose event $E$ is defined as in Equation~\eqref{eqn:e}. Then there is an event $F$ such that
$\P[F|E] \geq 1-\delta_k/2$, and on event $F$, for all $w$ in $W_k$,
\begin{eqnarray*}
\left| \E_{S_k} \ell_{\tau_k}(w, (x, y)) - \E_{D_k |_{T_k}} \ell_{\tau_k}(w, (x, y)) \right|
&\leq& \cfo  \cdot \ln\frac{n_k d}{\epsilon \delta_k} \cdot \sqrt{\frac{t(2\ln d + \ln \frac 2 {\delta_k})}{n_k}},
\end{eqnarray*}
\label{lem:sktk}
for some constant $\cfo > 0$ that depends on $c_2$ and $c_3$.
\end{lemma}

\begin{proof}
	Conditioned on event $E$, sample $S_k$ can be seen as drawn iid from $D_k|_{T_k}$.
	We consider the cases of $k = 0$ and $k \geq 1$ separately.
	\paragraph{Case 1: $k = 0$.} Using Corollary 4 of~\cite{KST09} with $\ell = \pm \ell_{\tau_0}$, $L_\ell = \frac{1}{\tau_0}$, $X = R_0$ and $W_1 = \sqrt{t}$ in the notations therein, we
	get that there is an event $F$, such that $\P[F|E] \geq 1-\delta_0/2$, on which for all $w$ in $W_k$,
	\begin{equation}
	\left| \E_{S_0} \ell_{\tau_0}(w, (x, y)) - \E_{D |_{T_0}} \ell_{\tau_0}(w, (x, y)) \right| \leq \frac{\cse}{\tau_0} \ln(\frac{2n_0 d}{\delta_0} \max(\frac9{\cf}, \frac{1}{\cf b_0})) \cdot \sqrt{\frac{32 t(\ln d + \ln \frac 4 {\delta_0})}{n_0}}.  \nonumber
	\end{equation}

	\paragraph{Case 2: $k \geq 1$.} By Lemma~\ref{lem:rad} below, we have that there is an event $F$, such that $\P[F|E] \geq 1-\delta_k/2$, on which for some constant $C_{10} > 0$ and for all $w$ in $W_k$,
	\begin{eqnarray}
	&& \left| \E_{S_k} \ell_{\tau_k}(w, (x, y)) - \E_{D_k |_{T_k}} \ell_{\tau_k}(w, (x, y)) \right| \nonumber \\
	&\leq& (1 + \frac{b_k}{\tau_k} + \frac{\rho_k R_k}{\tau_k}) \sqrt{\frac{\ln d + \ln \frac 2 {\delta_k}}{n_k}} \nonumber \\
	&\leq& C_{10} \cdot \ln\frac{n_k d}{\epsilon \delta_k} \cdot \sqrt{\frac{t(2\ln d + \ln \frac 2 {\delta_k})}{n_k}}, \nonumber
	\end{eqnarray}
	where the second inequality is by observing that $\frac{b_k}{\tau_k} = \frac{c_2}{c_3}$ and $\frac{\rho_k}{\tau_k} = \frac{\sqrt{2t}}{8c_3}$ and recalling that $R_k = \cse \ln(\frac{2n_k d}{\delta_k} \max(\frac9{\cf }, \frac{1}{\cf b_k}))$.

	Combining the above two cases, we can find a large enough constant $\cfo >0$ such that the lemma statement holds.
\end{proof}

We next show Lemma~\ref{lem:rad}, a key concentration result used in the proof of Lemma~\ref{lem:sktk}.
\begin{lemma}
Given $k$ in $\{1,\ldots,k_0\}$, suppose $S_k$ is a set of $n_k$ iid samples drawn from $D_k |_{T_k}$. We have that with probability $1-\delta_k/2$,
for all $w$ in $W_k$,
\[
|\E_{S_k} \ell_{\tau_k}(w, (x, y)) - \E_{D_k} \ell_{\tau_k}(w, (x, y))|
\leq
(1 + \frac{b_k}{\tau_k} + \frac{\rho_k R_k}{\tau_k}) \sqrt{\frac{2 \ln d + \ln \frac 2 {\delta_k}}{n_k}}.
\]
\label{lem:rad}
\end{lemma}
\begin{proof}
First, for all $w$ in $W_k$, $(x,y) \in T_k$, the instantaneous hinge loss $\ell_{\tau_k}(w, (x, y))$ is at most $1+\frac{|w \cdot x|}{\tau_k} \leq 1+\frac{|w_{k-1} \cdot x|}{\tau_k}+\frac{|(w-w_{k-1}) \cdot x|}{\tau_k} \leq 1 +\frac{b_k}{\tau_k} + \frac{\rho_k R_k}{\tau_k}$.
By standard symmetrization arguments (see Theorem 8 of \cite{BM02}), we have that with probability $1-\delta_k/2$, for all $w$ in $W_k$,
\begin{equation}
|\E_{S_k} \ell_{\tau_k}(w, (x, y)) - \E_{D_k} \ell_{\tau_k}(w, (x, y))| \leq (1 + \frac{b_k}{\tau_k} + \frac{\rho_k R_k}{\tau_k})\sqrt{\frac{\ln \frac 2 {\delta_k}}{2 n_k}} + R_{n_k}(\calF),
\label{eqn:hoeff}
\end{equation}
where $R_{n_k}(\cdot)$ denotes the Rademacher complexity over the examples in $S_k$,
$\calF$ is the set of functions $\{(x,y) \mapsto (1- \frac{y w \cdot x}{\tau_k})_+: w \in W_k \}$.
Note that $\calF$ can be written as the composition of $\phi(a):= (1- \frac{a}{\tau_k})_+$ and function class
$\calG := \{(x,y) \mapsto y w \cdot x: w \in W_k \}$.

By the contraction inequality of Rademacher complexity (see Theorem 12 of \cite{BM02}) and the $\frac{1}{\tau_k}$-Lipschitzness of $\phi$,
$R_{n_k}(\calF)$ is at most $\frac{1}{\tau_k} R_{n_k}(\calG)$. We now focus on bounding $R_{n_k}(\calG)$. First,
denote by $(x_i, y_i)$, $i=1,\ldots,n_k$ the elements of $S_k$. By the definition of Rademacher complexity,
\[ R_{n_k}(\calG) = \frac 1 {n_k} \E_\sigma \sup_{w \in W_k} \sum_{i=1}^{n_k} \sigma_i y_i w \cdot x_i,  \]
where $\sigma = (\sigma_1, \ldots, \sigma_{n_k})$, $\sigma_i$'s are iid random variables that take values uniformly in $\{-1,+1\}$.

It can be easily seen that $\sigma$ has the same distribution as $(\sigma_1 y_1, \ldots, \sigma_{n_k} y_{n_k})$. Hence, $R_n(\calG)$ can be simplified to
\[ R_{n_k}(\calG) = \frac 1 {n_k} \E_\sigma \sup_{w \in W_k} \sum_{i=1}^{n_k} \sigma_i w \cdot x_i.  \]

We bound $R_{n_k}(\calG)$ as follows:
\begin{eqnarray*}
R_{n_k}(\calG) &\leq& \frac 1 {n_k} \E_\sigma \sup_{w: \| w - w_{k-1} \|_1 \leq \rho_k} \sum_{i=1}^{n_k} \sigma_i w \cdot x_i \\
&=& \frac 1 {n_k} \E_\sigma \sup_{v: \| v \|_1 \leq \rho_k} \sum_{i=1}^{n_k} \sigma_i (w_{k-1} \cdot x_i + v \cdot x_i) \\
&=& \frac 1 {n_k} \E_\sigma \sup_{v: \| v \|_1 \leq \rho_k} \sum_{i=1}^{n_k} \sigma_i v \cdot x_i + \frac 1 {n_k} \E_\sigma \sum_{i=1}^{n_k} \sigma_i w_{k-1} \cdot x_i, \\
\end{eqnarray*}
where the inequality uses the fact that all $w$'s in $W_k$ satisfy that $\| w - w_{k-1} \|_1 \leq \rho_k$.

As all $x_i$'s have $\ell_\infty$ norm at most $R_k$, by Theorem 1, Example 2 of ~\cite{KST09}, the first term is bounded by
$\rho_k \cdot R_k \cdot \sqrt{\frac{2 \ln d}{n_k}}$. In addition, as all $(x_i, y_i)$'s are sampled from $D_k$, for all $i$, $|w_{k-1} \cdot x_i| \leq b_k$. Therefore, the second term can be bounded by:
\[ \frac 1 {n_k} \E_\sigma \sum_{i=1}^{n_k} \sigma_i w_{k-1} \cdot x_i \leq \frac 1 {n_k} \sqrt{\E_\sigma \left(\sum_{i=1}^{n_k} \sigma_i w_{k-1} \cdot x_i \right)^2} \leq b_k \sqrt{\frac 1 {n_k}}. \]
Summing the two bounds up, we have that $R_{n_k}(\calG) \leq (b_k + \rho_{k} R_k)\sqrt{\frac{2 \ln d}{n_k}}$. Therefore,
\[ R_{n_k}(\calF) \leq (\frac{b_k}{\tau_k} + \frac{\rho_{k}}{\tau_k} R_k)\sqrt{\frac{2 \ln d}{n_k}}. \]
Combining this inequality with Equation~\eqref{eqn:hoeff}, along with some algebraic calculations, we get the lemma as stated.
\end{proof}

\begin{lemma}
	For any $c_2, c_3 > 0$, there is a constant $\cn  > 0$ such that for all
	$k$ in $\{0,1,\ldots,k_0\}$, $w$ in $W_k$,
	\begin{equation*}
	|\E_{D_k|_{T_k}} \ell_{\tau_k}(w, (x, y)) - \E_{D_k} \ell_{\tau_k}(w, (x, y))| \leq \cn \sqrt{\frac{1}{n_k}}.
	\end{equation*}
	\label{lem:tk}
\end{lemma}
\begin{proof}
We consider the cases of $k = 0$ and $k \geq 1$ separately.
\paragraph{Case 1: $k = 0$.} Observe that $\P_{D}((x,y) \notin T_0) \leq \frac{\delta_0}{2n_0} \leq \frac{1}{n_0}$, and $\E_{D}(w \cdot x)^2 \leq 1$ for $w$ in $W_0$ as $D$ is isotropic.
Using Lemma~\ref{lem:conditional}, this implies that
\begin{equation*}
\left|\E_{D |_{T_0}} \ell_{\tau_0}(w, (x, y)) - \E_{D} \ell_{\tau_0}(w, (x, y)) \right| \leq 6\sqrt{\frac{1}{n_0}\left(1 + \frac{1}{c_3^2}\right)}.
\end{equation*}

\paragraph{Case 2: $k \geq 1$.} Observe that by Lemma~\ref{lem:variance}, there is a constant $\ce $ such that for all $w$ in $W_k \subset \{w \in \R^d: \| w - w_{k-1}\|_2 \leq r_k\}$,
$\E_{D_k}(w \cdot x)^2 \leq \ce  (b_k^2 + r_k^2)$. In addition, $\P_{D_k}((x,y) \notin T_k) \leq \frac 1 {n_k}$. Therefore, by Lemma~\ref{lem:conditional} and the definitions of $b_k$, $r_k$ and $\tau_k$, we have
\begin{equation*}
|\E_{D_k|_{T_k}} \ell_{\tau_k}(w, (x, y)) - \E_{D_k} \ell_{\tau_k}(w, (x, y))| \leq 6 \sqrt{\frac{1}{n_k} \left(1 + \frac{\ce (b_k^2 + r_k^2)}{\tau_k^2}\right)} = 6 \sqrt{\frac{1}{n_k} \left(1 + \frac{\ce}{c_3^2}(\frac 1 {64} + c_2^2)\right)}.
\end{equation*}

Combining the above two cases, we can find a large enough constant $\cn >0$ such that the lemma statement holds.
\end{proof}

In the proof of Lemma~\ref{lem:tk}, we use the following lemma to bound the difference between $\E_{D_k |_{T_k}} \ell_{\tau_k}(w, (x, y))$ and
$\E_{D_k} \ell_{\tau_k}(w, (x, y))$ in terms of $T_k$'s probability mass in $D_k$ and $D_k$'s second moments.
\begin{lemma}
For $k$ in $\{0,1,\ldots,k_0\}$, if $\P_{D_k}((x,y) \notin T_k) \leq \frac {\delta_k} {2n_k}$, then the following inequality holds for all $w$ in $\R^d$:
\begin{eqnarray*}
&&\left|\E_{D_k |_{T_k}} \ell_{\tau_k}(w, (x, y)) - \E_{D_k} \ell_{\tau_k}(w, (x, y)) \right|
\leq
 6 \sqrt{\P_{D_k}((x,y) \notin T_k)} \cdot \sqrt{1 + \frac{\E_{D_k}(w \cdot x)^2}{\tau_k^2}}.
\end{eqnarray*}
\label{lem:conditional}
\end{lemma}
\begin{proof}
First, observe that
\begin{equation}
\E_{D_k} \ell_{\tau_k}(w, (x, y)) = \E_{D_k |_{T_k}} \ell_{\tau_k}(w, (x, y)) \P_{D_k}((x,y) \in T_k) + \E_{D_k} \ell_{\tau_k}(w, (x, y)) I((x,y) \notin T_k).
\label{eqn:decomp}
\end{equation}

Therefore,
\begin{eqnarray}
&& \left|\E_{D_k |_{T_k}} \ell_{\tau_k}(w, (x, y)) - \E_{D_k} \ell_{\tau_k}(w, (x, y)) \right| \nonumber \\
&=& \left|\frac{\P_{D_k}((x,y) \notin T_k)}{\P_{D_k}((x,y) \in T_k)} \E_{D_k} \ell_{\tau_k}(w, (x, y)) - \frac{\E_{D_k} \ell_{\tau_k}(w, (x, y)) I((x,y) \notin T_k)}{\P_{D_k}((x,y) \in T_k)}  \right| \nonumber \\
&\leq& 2 \P_{D_k}((x,y) \notin T_k) \E_{D_k} \ell_{\tau_k}(w, (x, y)) + 2 \E_{D_k} \ell_{\tau_k}(w, (x, y)) I((x,y) \notin T_k) \nonumber \\
&\leq& 2 \P_{D_k}((x,y) \notin T_k) \E_{D_k} (1 + \frac{|w \cdot x|}{\tau_k}) + 2 \E_{D_k} (1 + \frac{|w \cdot x|}{\tau_k}) I((x,y) \notin T_k) \nonumber \\
&\leq& 2 \P_{D_k}((x,y) \notin T_k) (1 + \sqrt{\frac{\E_{D_k}(w \cdot x)^2}{\tau_k^2}}) + 2 \sqrt{\P_{D_k}((x,y) \notin T_k) \E_{D_k}(1 + \frac{(w \cdot x)}{\tau_k})^2 } \nonumber\\
&\leq& 6 \sqrt{\P_{D_k}((x,y) \notin T_k)} \cdot \sqrt{1 + \frac{\E_{D_k}(w \cdot x)^2}{\tau_k^2}}, \nonumber
\end{eqnarray}
where the equality is from Equation~\eqref{eqn:decomp} and algebra; the first inequality is from that $\P_{D_k}((x,y) \in T_k) \geq 1 - \frac{\delta_k}{2n_k} \geq \frac 1 2$ and the elementary inequality $|a+b| \leq |a|+|b|$;
the second inequality is from that $\ell_{\tau_k}(w,(x,y)) \leq (1 + \frac{|w \cdot x|}{\tau_k})$;
the third inequality is by applying Cauchy-Schwarz on both terms, and the last inequality is from algebra (using the following elementary inequalities: $\sqrt{a} + \sqrt{b} \leq \sqrt{2(a+b)}$, $\P_{D_k}((x,y) \notin T_k) \leq 1$ and $(a+b)^2 \leq 2(a^2+b^2)$).
\end{proof}

\section{Auxiliary lemmas}
The lemmas in this section are known and used in previous works on efficient halfspace learning under isotropic log-concave distributions~\citep[See e.g.][]{ABL17, ABHZ16}; we collect them here for completeness.

The following lemma characterizes one-dimensional projections of isotropic log-concave distributions (which are in fact also isotropic log-concave).
\begin{lemma}[\citet{LV07}]
There exists a numerical constant $\cf \in (0,1)$ such that the following holds.
Given
a unit vector $v$ and a positive real number $b$,
\[ \min(\cf /9, \cf  b) \leq \P_{D_X}(|v \cdot x| \leq b ) \leq 9 b. \]
\label{lem:bandmass}
\end{lemma}

Suppose $w$ is a unit vector, and $B = \{w: |w \cdot x| \leq b \}$ is a band of width
$b > 0$ along the $w$ direction.
The following technical lemma bounds the second moments of
$D_X|_B$, along directions close to $w$.

\begin{lemma}[\citet{ABL17}]
Suppose $w, b, B$ are defined as above. Then there is a numerical constant $\ce  > 0$, such that for all $w' \in \{v: \| v - w \|_2 \leq r \}$, we have
\begin{equation*}
 \E_{D_X|_B} (w' \cdot x)^2 \leq \ce (r^2 + b^2).
\end{equation*}
\label{lem:variance}
\end{lemma}

Recall that $D$ (resp. $\tilde{D}$) is the joint distribution over $(x,y)$ (resp. $(x,\sign(u \cdot x))$). In addition, recall that $b_k = c_2 2^{-k}$, $\tau_k = c_3 2^{-k}$ and $B_k = \{ x: |w_{k-1} \cdot x| \leq b_k \}$.
The following lemma shows that under certain ``local low noise'' conditions on $D$, for every halfspace $w$ in $W_k$,
its expected $\tau_k$-hinge loss on $D_k$ is close to that on $\tilde{D}_k$. With the help of this result, in Lemmas~\ref{lem:tv-an} and~\ref{lem:tv-bn}, we will show that under the $t$-sparse $\mu_1 \epsilon$-adversarial noise condition and $t$-sparse $\mu_2$-bounded noise condition for sufficiently small $\mu_1$ and $\mu_2$,
the hinge loss of $w$ on $D_k$ is at most a constant away from the hinge loss of $w$ on $\tilde{D}_k$,
for all $w$ in $W_k$.

\begin{lemma}
For any choice of $c_2, c_3 > 0$, there exists a constant
$\ct  > 0$ such that the following holds. For every $k$ in $\{0,1,\ldots,k_0\}$, suppose $\P_{D_k}(y \neq \sign(u \cdot x)) \leq \xi_k$, then for every
$w \in W_k$,
\begin{equation}
| \E_{D_k} \ell_{\tau_k}(w, (x, y)) - \E_{\tilde{D}_k} \ell_{\tau_k}(w, (x, y)) |
\leq
\sqrt{\ct  \xi_k}.
\label{eqn:lossdiff}
\end{equation}
\label{lem:noisy-hinge}
\end{lemma}
\begin{proof}
We first bound the difference as follows:
\begin{eqnarray}
&& |\E_{D_k} \ell_{\tau_k}(w, (x, y)) - \E_{\tilde{D}_k} \ell_{\tau_k}(w, (x, y)) | \nonumber \\
&=& |\E_{D_k} [\ell_{\tau_k}(w, (x, y)) - \ell_{\tau_k}(w, (x, \sign(u \cdot x)))] | \nonumber \\
&=& | \E_{D_k} I(y \neq \sign(u \cdot x)) \cdot (\ell_{\tau_k}(w, (x, y)) - \ell_{\tau_k}(w, (x, \sign(u \cdot x)))) | \nonumber \\
&\leq& \E_{D_k} I(y \neq \sign(u \cdot x)) \cdot 2\frac{|w \cdot x|}{\tau_k} \nonumber \\
&\leq& 2 \sqrt{ \P_{D_k}(y \neq \sign(u \cdot x)) \frac{\E_{D_k} (w \cdot x)^2 }{\tau_k^2} }
\label{eqn:cs}
\end{eqnarray}
where the first inequality is from that an example $(x,y)$ drawn from $\tilde{D}_k$ satisfies
$y = \sign(u \cdot x)$ with probability 1; the second inequality is by decomposing $1$ as
 $I(y \neq \sign(u \cdot x)) + I(y = \sign(u \cdot x))$; the first inequality is from that
$|(1+\frac{|w \cdot x|}{\tau_k})_+ - (1-\frac{|w \cdot x|}{\tau_k})_+| \leq 2\frac{|w \cdot x|}{\tau_k}$;
the second inequality is from Cauchy-Schwarz. We now consider the cases of $k=0$ and $k \geq 1$ respectively.
\paragraph{Case 1: $k = 0$.} In this case, $W_0$ is a subset of $\{w \in \R^d: \| w \|_2 \leq 1\}$ and $D_0 = D$.
Therefore, for all $w$ in $W_0$,
\[ \E_{D_0} (w \cdot x)^2 \leq 1 \]
as $D$ is isotropic log-concave. Continuing Equation~\eqref{eqn:cs}, we get that
\[ |\E_{D_0} \ell_{\tau_0}(w, (x, y)) - \E_{\tilde{D}_0} \ell_{\tau_0}(w, (x, y)) | \leq  2\sqrt{\frac{\xi_0}{\tau_0^2}} = 2\sqrt{\frac{\xi_0}{c_3^2}}. \]

\paragraph{Case 2: $k \geq 1$.} In this case, $W_k$ is a subset of $\{w \in \R^d: \| w - w_{k-1} \|_2 \leq r_k \}$.
By Lemma~\ref{lem:variance}, and the choices of $b_k$ and $r_k$, we have that for all $w$ in $W_k$,
\begin{equation}
 \E_{D_k} (w \cdot x)^2 \leq \ce (r_k^2 + b_k^2).
\label{eqn:variance}
\end{equation}
By the definitions of $r_k$, $b_k$, and $\tau_k$ and Equation~\eqref{eqn:cs}, we have
\[
|\E_{D_k} \ell_{\tau_k}(w, (x, y)) - \E_{\tilde{D}_k} \ell_{\tau_k}(w, (x, y)) |
\leq
2 \sqrt{\xi_k \frac{\ce(r_k^2 + b_k^2)}{\tau_k^2}} = 2 \sqrt{\xi_k \ce  \frac{1+c_2^2}{c_3^2}}.
\]

Now, choose $\ct = \max\left(\ce(\frac{1+c_2^3}{c_3^2}), \frac{1}{c_3^2}\right)$. Combining the above two cases and by the choice of $\ct $, we conclude that Equation~\eqref{eqn:lossdiff} holds for all $k$ in $\{0,1,\ldots,k_0\}$.
\end{proof}

Applying the above lemma to the two noise settings respectively, we have:

\begin{lemma}
For any $\lambda > 0$ and $c_2, c_3 > 0$, there exists a constant $\mu_1 > 0$ such that the following holds. Suppose $D$ satisfies the $t$-sparse $\mu_1 \epsilon$-bounded
noise condition. For every $k \in \{0,\ldots,k_0\}$, and $w$ in $W_k$,
\[
| \E_{D_k} \ell_{\tau_k}(w, (x, y)) - \E_{\tilde{D}_k} \ell_{\tau_k}(w, (x, y)) |
\leq
\lambda.
\]
\label{lem:tv-an}
\end{lemma}
\begin{proof}
By Lemma~\ref{lem:noisy-hinge}, it suffices to let $\mu_1$ be such that for all $k \in \{0,1,\ldots,k_0\}$, $\P_{D_k}(y \neq \sign(u \cdot x)) \leq \frac{\lambda^2}{\ct }$ for the $\ct$ defined therein.
Observe that
\[ \P_{D_k}(y \neq \sign(u \cdot x)) \leq \frac{\P_D(y \neq \sign(u \cdot x))}{\P_D(x \in B_k)}. \]
by the fact that $\P(A|B) \leq \frac{\P(A)}{\P(B)}$ for any two events $A$, $B$.
We now consider the cases of $k=0$ and $k \geq 1$ respectively.

\paragraph{Case 1: $k = 0$.} In this case, $B_k = \R^d$, hence $\P_D(x \in B_k) = 1$. It suffices to set $\mu_1 \leq \frac{\lambda^2}{\ct }$.

\paragraph{Case 2: $k \geq 1$.} In this case, $\P_D(x \in B_k) \geq \min(\cf /9, \cf  b_k) \geq \min(\cf /9, \cf  c_2 C_1 \epsilon / 2)$, where the first inequality is from Lemma~\ref{lem:bandmass}; the second inequality is from the definition of $b_k$ and $k \leq k_0$.
Therefore, for sufficiently small $\mu_1$, if $\P_D(y \neq \sign(u \cdot x)) \leq \mu_1 \epsilon$, then $\P_{D_k}(y \neq \sign(u \cdot x)) \leq \frac{2 \mu_1}{\min(2\cf /9, \cf  C_1 c_2)} \leq \frac{\lambda^2}{\ct }$.

Combining the above two cases, we can pick a sufficiently small $\mu_1$ such that the requirements on $\mu_1$ in both cases are satisfied. This completes the proof.
\end{proof}

\begin{lemma}
For any $\lambda > 0$ and $c_2, c_3 > 0$, there exists a constant $\mu_2 > 0$ such that the following holds. Suppose $D$ satisfies the $t$-sparse $\mu_2$-bounded
noise condition. For every $k \in \{0,\ldots,k_0\}$, and $w$ in $W_k$,
\[
| \E_{D_k} \ell_{\tau_k}(w, (x, y)) - \E_{\tilde{D}_k} \ell_{\tau_k}(w, (x, y)) |
\leq
\lambda.
\]
\label{lem:tv-bn}
\end{lemma}
\begin{proof}
By Lemma~\ref{lem:noisy-hinge}, it suffices to let $\mu_2$ be such that for all $k \in \{0,1,\ldots,k_0\}$, $\P_{D_k}(y \neq \sign(u \cdot x)) \leq \frac{\lambda^2}{\ct }$ for the $\ct$ defined therein.
This can indeed be satisfied by setting $\mu_2 = \frac{\lambda^2}{\ct }$, which immediately implies that $\P_{D_k}(y \neq \sign(u \cdot x)) \leq \mu_2 \leq \frac{\lambda^2}{\ct }$.
\end{proof}


\end{document}